\documentclass[lettersize,journal]{IEEEtran}
\usepackage{amsmath,amsfonts}
\usepackage{algorithmic}
\usepackage{algorithm}
\usepackage{array}
\usepackage{textcomp}
\usepackage{stfloats}
\usepackage{url}
\usepackage{verbatim}
\usepackage{graphicx}
\usepackage{cite}
\usepackage{bm}
\usepackage{tablefootnote}
\usepackage[capitalise]{cleveref}
\usepackage{tablefootnote}
\usepackage{fp,tikz,pgfplots}
\usetikzlibrary{arrows,shapes,backgrounds,patterns,fadings,matrix,arrows,calc,
	intersections,decorations.markings,
	positioning,arrows.meta}
\usepgfplotslibrary{fillbetween}
\usepgfplotslibrary{statistics}
\usepgfplotslibrary{groupplots}
\usepackage{xcolor}
\usepackage{caption}
\usepackage{subcaption}
\pgfplotsset{width=5\columnwidth /5, compat = 1.13,
	height = 60\columnwidth /100, grid= major,
	legend cell align = left, ticklabel style = {font=\scriptsize},
	every axis label/.append style={font=\small},
	legend style = {font={\scriptsize}},title style={yshift=-7pt, font = \small} }

\newtheorem{assumption}{\bf{Assumption}}
\newtheorem{theorem}{\bf{Theorem}}

\newtheorem{lemma}{\bf{Lemma}}
\newtheorem{remark}{\bf{Remark}}
\newtheorem{definition}{\bf{Definition}}

\xdefinecolor{redZ}{RGB}{234,29,93}
\xdefinecolor{blueZ}{RGB}{19,106,213}
\xdefinecolor{yellowZ}{RGB}{251,138,46}

\xdefinecolor{alienware_black}{RGB}{34, 34, 34}
\xdefinecolor{alienware_red}{RGB}{206,24,30}
\xdefinecolor{alienware_blue}{RGB}{0,124,192}
\xdefinecolor{alienware_yellow}{RGB}{255,194,14}
\xdefinecolor{carrot_blue}{RGB}{56,142,209}
\xdefinecolor{carrot_orange}{RGB}{215,78,38}
\xdefinecolor{carrot_green}{RGB}{0,104,55}
\xdefinecolor{carrot_yellow}{RGB}{247,147,30}
\newcommand{\bandcurve}[5]{%
	\addplot[#1, each nth point=2] table[x=#4, y=mean_#5]{#3};
	\addplot[name path=#5U, draw=none, forget plot, each nth point=3] table[x=#4, y=hi_#5]{#3};
	\addplot[name path=#5L, draw=none, forget plot, each nth point=3] table[x=#4, y=lo_#5]{#3};
	\addplot[#2, fill opacity=0.5, forget plot] fill between[of=#5U and #5L];}
\definecolor{carrot_orange}{RGB}{230,126,34}
\definecolor{carrot_blue}{RGB}{41,128,185}
\definecolor{carrot_green}{RGB}{39,174,96}
\definecolor{carrot_yellow}{RGB}{246,190,0}
\definecolor{c_red}{RGB}{192,57,43}
\definecolor{c_purple}{RGB}{142,68,173}

\begin{document}
	
	\title{
		COIN-GP: Cooperative Online Learning in Networked Distributed Systems with Partial Measurements via Gaussian Process Regression
	}
	
	\author{
		Zewen Yang$^{1}$,~\IEEEmembership{Member,~IEEE}, 
		Xiaobing Dai*$^{2}$, 
		Zhenxiao Yin$^{3}$, 
		Hang Zhao$^{3}$, \\
		Zhijun Li$^{4}$,~\IEEEmembership{Fellow,~IEEE},
		C.C. Chan$^{5}$,~\IEEEmembership{Life Fellow,~IEEE}
		
		\thanks{*Corresponding author: Xiaobing Dai \textless{}xiaobing.dai@tum.de\textgreater{}}
		\thanks{
			$^{1}$Zewen Yang is with the Chair of Robotics and Systems Intelligence (RSI), Munich Institute of Robotics and Machine Intelligence (MIRMI), Technical University of Munich (TUM), 80992 Munich, Germany.
		}
		\thanks{
			$^{2}$Xiaobing Dai is with the Technical University of Munich, 80333 Munich, Germany.
		}
		\thanks{
			$^{3}$Zhenxiao Yin and Hang Zhao are with The Hong Kong University of Science and Technology (Guangzhou).
		}
		\thanks{
			$^{4}$Zhijun Li is with the School of Mechanical Engineering, Translational Research Center, Tongji University, Shanghai 201804, China, affiliated with Shanghai Yangzhi Rehabilitation Hospital and also with Shanghai Key Laboratory of Wearable Robotics and Human-Machine Interaction, and also with Department of Automation, University of Science and Technology of China, Hefei 230026, China.
		}
		\thanks{
			$^{5}$ C.C. Chan is with The Hong Kong Polytechnic University.
		}
		
	}
	
	\maketitle
	
	\begin{abstract}
		In this paper, we tackle the problem of jointly estimating the system states and partially unknown dynamics within distributed sensor-equipped networks, particularly in scenarios where only partial state observations are available.
		To address this issue, we propose an observer-based dynamic cooperative learning framework incorporating online distributed Gaussian Process (GP) regression, which enables accurate estimation despite incomplete in measurements and deficient GP models. 
		In addition, a novel data collection strategy is introduced, with theoretical conditions ensuring feasible data acquisition. 
		Moreover, we also derive an error upper bound encompassing state estimation and model estimation, leveraging the deterministic error bounds of GPs. 
		Empirical simulations demonstrate the superiority of our approach compared to existing distributed GP-based methods.
	\end{abstract}
	
	
	\section{Introduction}
	\IEEEPARstart{C}{ooperative}  learning in distributed dynamical systems has gained significant attention due to its broad applications in areas such as sensor networks, multi-robot systems, and autonomous underwater vehicles~\cite{Huang_TCYB2025_Visual,Ma_TCYB2024_Adaptive,Yan_OE2020_Virtual}. 
	The primary goal of cooperative learning is to enable a group of nodes or agents to collaboratively infer the unknown pattern by sharing information and leveraging neighbors' learned results. 
	
	In this study, we investigate the problem of handling partial measurements in distributed learning systems while ensuring performance guarantees. 
	The major challenge lies in the joint estimation of the system states and the unknown function. 
	Although state estimation techniques such as the Luenberger observer, high-gain observer, and sliding mode observer have been widely employed in feedback systems, they face significant limitations for discrete-time systems with unknown nonlinear dynamics components.
	In particular, the Luenberger observer, despite its efficacy in reconstructing states from partial measurements, relies on a fully known system model, rendering it impractical when the nonlinear component is unknown, even if the linear part of the dynamics is known~\cite{Zeitz_1987SCL_extendedLuenberger}.
	Similarly, high-gain observers, primarily designed for continuous-time systems, encounter stability issues and exacerbated peaking phenomena when discretized, due to finite sampling rates~\cite{Khalil_IJRNC2014_High}. 
	Sliding mode observers, while robust in continuous-time applications, struggle with the implementation of infinitely fast switching in discrete-time systems, leading to chattering and reduced accuracy~\cite{Spurgeon_2008IJSS_Sliding}. 
	To address these challenges, we employ Gaussian process (GP) regression, a machine learning technique particularly well-suited for unknown function modeling in dynamical systems~\cite{Rasmussen_2006_Gaussian}. 
	As a Bayesian learning method, GP regression is widely recognized for its rigorous prediction error quantification~\cite{Maddalena_2021Auto_Deterministic} and is frequently used in model identification. 
	Its capability for online learning offers a distinct advantage, enabling real-time model updates based on new measurements and thereby facilitating adaptation to changing environments. 
	In addition to the problems posed by the partial measurements and unknown nonlinear dynamics components, we consider a distributed sensor network comprising multiple subsystems.
	These subsystems are equipped with sensors that measure only partial states of the overall system, adding another layer of complexity to the estimation process. 
	Furthermore, we consider the scenario where only a limited number of subsystems have access to sufficient training data for learning.
	Our proposed approach for cooperative online learning in networked distributed systems under partial measurements via GP regression (COIN-GP) enables the exchange of predictions among interconnected subsystems, which not only enhances the efficiency of the learning process but also improves the overall state estimation performance, particularly for subsystems with limited or poor datasets. 
	Through dynamic cooperative learning, subsystems can leverage information from their neighbors, significantly enhancing prediction accuracy and system robustness.

	\subsection{Related Work}
	In addressing the joint estimation problem within nonlinear systems, \cite{Bin_TAC2021_Model} introduced an adaptive state observer capable of both state estimation and model identification. 
	Building upon this framework, \cite{Buisson_ACC2021_Joint} further enhanced it by integrating GP models. 
	These observers rely on high-gain mechanisms, which present limitations when applied to discrete-time nonlinear systems as mentioned above. 
	Nevertheless, the existing methodologies primarily focus on single systems, thereby presenting challenges when extending their application to distributed systems, such as multi-agent systems (MASs) or multi-robot systems, due to the complexities inherent in interconnections between subsystems.
	
	In MASs, distributed observers are essential for estimating unknown dynamics or states. 
	Adaptive distributed observers have been proposed for leader-following consensus problems~\cite{Cai_TAC2016_Leader, Yu_TAC2016_Observer, Wang_TYCB2022_Designing}, including scenarios with uncertain leaders in a linear form~\cite{Wang_TAC2020_Adaptive} and reduced-order design~\cite{Wang_TYCB2022_Consensus}. 
	For formation tracking, distributed extended state observers have been used to model the nonlinear unknown functions of agents~\cite{Lv_IJRNC2023_Fully}, but these proposed methods are limited to second-order strict-feedback dynamics. 
	Even with significant progress in the development of observers for MASs, current methods primarily address either state estimation or unknown function modeling separately, lacking strategies for joint estimation problems in MASs. 
	To fill this gap, recent work has explored the use of neural networks for learning unknown functions in MASs~\cite{Tahoun_ISAT2022_Adaptive}. 
	For instance, learning-based distributed observers have been proposed to estimate both the states and parameters of the leader node~\cite{Wang_Auto2024_Learning, Wang_IS2024_Neural}. 
	However, these approaches often rely on the assumption that the unknown function is in a vector form of a linear combination of known features. 
	GP-based methods provide deterministic error bounds and naturally handle agents with limited or no training data through uncertainty-weighted cooperation, whereas NN-based approaches can offer greater scalability and capacity at the cost of interpretability and formal guarantees.

	Therefore, we consider GP regression for handling unknown dynamics that do not presuppose a predefined model structure. 
	In order to adapt to the distributed systems, distributed Gaussian processes (DGPs), where a central entity aggregates predictions from distributed experts using divided sub-datasets \cite{Bach_ICML2015_Distributed}. 
	Advanced methods for DGPs in MASs include optimizing aggregation weights through variations in the product of experts (PoE) method~\cite{yang_CDC2021_Distributed} and applying the dynamical average consensus algorithm~ \cite{Lederer_TAC2023_Cooperative} or maximal information coefficient~\cite{Ge_TYCB2025_Structure} to expedite convergence in joint predictions. 
	However, these approaches encounter challenges when applied to scenarios involving online learning. 
	Specifically, the offline collected dataset, which does not update dynamically during the operation with new observations, limits their adaptability in environments that are dynamic. Recent research \cite{Hoang_AAAI2019_Collective, dai_EJC2024_Decentralized, dai_TNNLS2025_Cooperative} underscores the critical importance of integrating streaming data to enhance the adaptability of DGPs.
	However, these methodologies require complete state observation in order to work.  
	While current approaches are well-suited for continuous-time systems, their direct applicability to discrete-time systems is constrained due to different stability conditions.
	More recently, cooperative GP-based learning has been advanced by elective learning~\cite{Yang_JAAMAS2026_Whom, yang2026quality}, control under switching topologies~\cite{yang_ACC2024_cooperative}, and event-triggered online learning~\cite{dai_TNNLS2025_Cooperative}, while distributed impulsive observers handle aperiodically sampled outputs~\cite{Zhou_TCYB2025_Leader}, interconnected systems with input and output disturbances~\cite{Mu_TCYB2023_Robust}, and uncertainty-aware learning is strengthened in~\cite{Yang_TCYB2024_Cooperative, Liu_TCYB2024_Robust}.
	Nevertheless, the learning-based approaches assume full state measurements, whereas the estimation-oriented schemes do not simultaneously learn the unknown dynamics.
	Therefore, to the best of our knowledge, there exists no approach designed explicitly to resolve the problems associated with unknown dynamics and incomplete measurements in online distributed learning with GPs in discrete-time systems while also allowing subsystems without inference capabilities.
	
	\subsection{Contribution and Structure} 
	The summary of the contributions of the paper is listed below.
	\begin{itemize}
		\item Data acquisition strategy: For constructing the training dataset with limited measurements, the data collection strategy is proposed along with sufficient conditions of data collectability.
		
		\item Distributed learning-based estimation: The COIN-GP method is presented not only to facilitate full state estimation with incomplete observations but also to enable each individual system to infer the unknown function. Additionally, the proposed dynamic cooperative learning approach allows the distributed system to operate without the need to establish a GP model for every subsystem.
		
		\item Learning and control performance analysis: The joint error upper bound of the state observation and model prediction is derived by employing the deterministic error bound inherited from GP regression, which ensures guaranteed performance for safety-critical applications.
	\end{itemize}
	
	The paper is organized as follows: In Section \ref{sec_PP}, we lay out the preliminaries and define the problem setting. Section \ref{section_observer_based_cooperative_learning} then elaborates on the data collection strategy and the COIN-GP approach, along with a comprehensive performance analysis. In Section \ref{sec_simulation}, we demonstrate the effectiveness of our proposed method by conducting comparative analyses with existing distributed GP-based aggregation frameworks through simulation experiments. Finally, Section \ref{sec_conclusion} presents the conclusions drawn from this study.

	\section{Preliminaries and Problem Formulation}\label{sec_PP}
	
	\subsection{Notation and Graph Theory}
	We denote $\mathbb{R}$ by the set of real numbers and let $\mathbb{R}_{>0}$ represents positive reals $(0,\infty)$ and $\mathbb{R}_{\geq 0}$ non-negative reals $[0,\infty)$. 
	Natural numbers are symbolized by $\mathbb{N}$ and $\mathbb{N}_{0}$ with zero. $\mathbb{I}_{[a, b]}$ denotes the integer interval from $a$ to $b$. The vector $\bm{1}_N \in \mathbb{R}^N$ is defined as $\bm{1}_N = [1,\cdots,1]^T$ and the $N\times N$ identity matrix as $\boldsymbol{I}_N$. 
	Define the function $\lambda_i(\cdot)$ returns the $i$-th eigenvalue of a square matrix, while $\underline{\lambda}(\cdot)$ returns the minimum eigenvalue, and $\bar{\lambda}(\cdot)$ returns the maximum eigenvalue of the matrix.
	Additionally, we represent a diagonal matrix as $\mathrm{diag}(a_1, \dots, a_N)$ and concatenated block matrices diagonally as $\mathrm{blkdiag}(\boldsymbol{A}_1, \dots, \boldsymbol{A}_N)$. The operator $\otimes$ denotes the Kronecker product.

	The nodes (systems/agents/sensors) are connected through a communication network, which is described as an undirected graph $\mathcal{G} = \{ \mathcal{V}, \mathcal{E} \}$ with the vertex set $\mathcal{V}=\{1,2,\dots, N\}$ with $N \in \mathbb{N}$. 
	The edge set $\mathcal{E} \in \mathcal{V} \times \mathcal{V}$ depicts the connection between each node, where the pair $(i,j) \in \mathcal{E}$ indicates the existence of the communication channel from node $j$ to node $i$ for $\forall i,j \in \mathcal{V}$, vice versa. The adjacent matrix $\mathcal{A}$ is defined based on the topology $\mathcal{G}$, where $a_{i,j} = 1$ if $(j,i) \in \mathcal{E}$ and $a_{i,j} = 0$ otherwise. Moreover, we define the set $\mathcal{N}_i$ is the neighbor set of node $i$, i.e., $\mathcal{N}_i = \{ j \in \mathcal{V} | (j,i) \in \mathcal{E} \}$.

	\subsection{Problem Formulation}
	In this paper, we investigate the estimation of a class of discrete-time Lur'e systems~\cite{lur1944theory}. These systems, characterized by a feedback structure involving a linear dynamical component and a nonlinear element, are governed by the  dynamics
	\begin{equation}
		\label{eq_dynamics}
		\bm{x}_0(k+1) = \bm{A} \bm{x}_0(k) + \bm{B} \bm{f}(\bm{x}_0(k)), \quad \forall k \in \mathbb{N}_{0},
	\end{equation}
	where \(\bm{x}_0(k) \in \mathbb{R}^n\) denotes the system state at time step \(k\), with \(n \in \mathbb{N}\) representing the state dimension. 
	The matrices \(\bm{A} \in \mathbb{R}^{n \times n}\) and \(\bm{B} \in \mathbb{R}^{n \times m}\) are known reflecting the system structure, with \(m \in \mathbb{N}\) being the dimension of the nonlinear output. 
	However, the unknown nonlinear function $\bm{f}(\cdot): \mathbb{R}^n \to \mathbb{R}^m$ embodies the latent trends governing the evolution of the system, which needs to be inferred. 
	
	\begin{remark}
		The system described in \eqref{eq_dynamics} can also be interpreted as a discrete-time control-affine system with matched nonlinearities, expressed as
		$\bm{x}_0(k+1) = \bm{A} \bm{x}_0(k) + \bm{B} \left( \bm{u}(\bm{x}_0(k)) + \bm{f}(\bm{x}_0(k)) \right)$, where \(\bm{u}(\cdot): \mathbb{R}^n \to \mathbb{R}^m\) represents a state-dependent control input. 
		Employing a feedback controller $\bm{u} = \bm{K}\bm{x}_0(k)$, the system \eqref{eq_dynamics} is retrieved from $\bm{x}_0(k+1) = \tilde{\bm{A}} \bm{x}_0(k) + \bm{B}  \bm{f}(\bm{x}_0(k))$ with $\tilde{\bm{A}} = {\bm{A}}+ \bm{K}$.
		Furthermore, in the absence of linear dynamics, i.e., when \(\bm{A} = \mathbf{0}\), the system simplifies to a discrete-time nonlinear autonomous form: \(\bm{x}(k+1) = \bm{B} \bm{f}(\bm{x}(k))\). 
		Therefore, this versatility enables the framework in \eqref{eq_dynamics} to model a diverse array of dynamical systems, ranging from linear controlled systems to fully nonlinear autonomous ones, enhancing its applicability across various domains.
	\end{remark}

	In order to obtain the behaviors of the partially unknown system \eqref{eq_dynamics} over time, specifically concerning the states $\bm{x}_0$ and the dynamics $\bm{f}(\cdot)$, we assume a collective of $N\in \mathbb{R}_{>0}$ networked sensors is deployed. 
	This arrangement gives rise to a multi-agent or multi-robot configuration, where each agent or robot is equipped with its own sensor. 
	Considering the partial observation of $\boldsymbol{x}_0$, each subsystem's measurement is given by 
	\begin{align}
		\label{eqn_measurement}
		\bm{y}_i(k) = \bm{C}_i \bm{x}_0(k) + \bm{v}_i(k), && \forall i \in \mathcal{V},
	\end{align}
	where $\bm{y}_i \in \mathbb{R}^{p_i}$ with $p_i \in \mathbb{I}_{[1,n]}$ is considered as the system output of the $i$-th sensor with $\bm{C}_i \in \mathbb{R}^{p_i \times n}$ as known, such that the pair $(\bm{A}, \bm{C}_i)$ is observable.
	The measurement noise denoted by $\bm{v}_i$ is bounded by $\bar{v}_i \in \mathbb{R}_{>0}$, i.e., $\| \bm{v}_i(k) \| \le \bar{v}_i$. 
	
	\begin{remark}
		Notice that the dimension of the output states $\boldsymbol{y}_i$ in \eqref{eqn_measurement} is not strictly equal to the dimension of the state $\boldsymbol{x}_0$, which indicates the incompleteness of the observations obtained through the sensor measurements. 
	\end{remark}
	
	Given the challenging circumstance posed by the partial observations, the principal aim is to empower the multi-sensor distributed system with the capability to cooperatively infer both the system states $\bm{x}_0(k)$ and the unknown dynamics $\bm{f}(\cdot)$. 
	Let the state estimation of the system states $\bm{x}_0$ for sensor $i \in \mathcal{V}$ denote as $\bm{x}_i$ and function estimation $\bm{f}(\cdot)$ as $\hat{\bm{f}}_i(\cdot)$, respectively. 
	Then the desired performance defined as practical convergence is formulated as follows.
	\begin{definition}\label{def_convergence}
		The networked sensor system is said to achieve jointly practical convergence for the observation of $\bm{x}_0$ and the inference of $\bm{f}(\cdot)$, if
		\begin{align}
			\| \bm{x}(k) \!- \!\bm{1}_N \!\otimes \!\bm{x}_0(k) \| \le \epsilon_x 
			\wedge &\| \hat{\bm{f}}(\bm{x}(k)) \!-\! \bm{1}_N \! \otimes \!\bm{f}(\bm{x}_0(k)) \| \le \epsilon_f \nonumber
		\end{align}
		for $k >0 \bar{k}$ with small constants $\epsilon_x, \epsilon_f \in \mathbb{R}_{>0}$ and $\bar{k} \in \mathbb{N}_0$, where $\bm{x} \!=\! [\bm{x}_1^T, \!\cdots\!, \bm{x}_N^T]^T$, $\hat{\bm{f}}(\bm{x}) \!=\! [\hat{\bm{f}}_1^T(\bm{x}_1), \!\cdots\!, \hat{\bm{f}}_N^T(\bm{x}_N)]^T$ are the concatenated estimation for $\bm{x}_0$ and $\bm{f}(\bm{x}_0)$, respectively.
	\end{definition}
	For ease of notation, we assume that  $\bm{f}(\cdot)$ is scalar, i.e., $m=1$. Consequently, the matrix $\bm{B}$ is simplified to a vector $\bm{b} \in \mathbb{R}^n$.
	Note that the derived results in the following sections can be extended to the multi-dimensional case using multi-output machine learning techniques, by assuming each output is independent~\cite{Xu_TNNLS2020_Survey,yang_ACC2024_cooperative}.
	
	\subsection{Gaussian Process Regression}
	In order to infer the unknown function $f(\cdot)$, Gaussian process regression is employed, inducing a Gaussian distribution over function characterized by a kernel function $\kappa(\cdot,\cdot): \mathbb{R}^n \times \mathbb{R}^n \to \mathbb{R}_{\ge0}$ and a mean function $m(\cdot): \mathbb{R}^n \to \mathbb{R}_{\ge0}$.
	Specifically, the representation is formalized as ${f}(\cdot) \sim \mathcal{GP} (m(\cdot), \kappa(\cdot,\cdot))$. 
	The mean function embodies prior knowledge of the unknown function and is commonly set to zero in the absence of observable trends of the system in advance. 
	Moreover, the kernel function delineates the covariance between two samples and adheres to the following assumption.
	\begin{assumption} \label{assumption_f}
		The unknown function $f$ belongs to the reproduced kernel Hilbert space (RKHS) $\mathcal{H}_{\kappa}$ corresponding to a given stationary and differentiable Lipschitz-continuous kernel $\kappa(\cdot,\cdot)$ with Lipschitz constant $L_{\kappa} \in \mathbb{R}_{>0}$ induced by Euclidean norm.
		Moreover, the kernel satisfies $\kappa(\bm{\xi}, \bm{\xi}) = \sigma_f^2$ and $\sigma_f \in \mathbb{R}_+$ for $\bm{\xi} \in \mathbb{R}^n$.
		The induced inner product defined on RKHS denotes $\langle \cdot, \cdot \rangle_{\kappa}$, and the RKHS norm defined as $\| f \|_{\kappa}^2 = \sqrt{\langle f, f \rangle_{\kappa}}$ is bounded by $\Gamma \in \mathbb{R}_+$, i.e., {$\| f \|_{\kappa} \le \Gamma$}.
	\end{assumption}
	The RKHS norm $\| \cdot \|_{\kappa}$ indicates the smoothness of the function $f(\cdot)$, such that the existence of the upper bound $\Gamma$ is equivalent to the bounded derivative of $f(\cdot)$ requiring $f(\cdot)$ to be Lipschitz.
	Note that the Lipschitz system dynamics in \eqref{eq_dynamics} is a common prerequisite in continuous-time systems in reality, guaranteeing the dynamics \eqref{eq_dynamics} has a unique solution \cite{khalil2015nonlinear}. 
	For instance, common nonlinear dynamics in robotics and control, such as friction models, aerodynamic drag, and smooth manipulator gravitational/Coriolis terms, are Lipschitz continuous over compact operating regions, ensuring that Assumption~\ref{assumption_f} is satisfied.
	The upper bound $\Gamma$ of the RKHS norm can be approximated using the method in \cite{hashimoto2022learning}. 
	More recently, data-driven methods for over-estimating the RKHS norm bound have been developed in \cite{ Tokmak_AISTAT2025_Safe}, which leverage neural network approximators and collected data.
	Therefore, Assumption~\ref{assumption_f} does not impose restrictive conditions and is broadly satisfied in real-world dynamical systems.
	
	As modeling $f$ using GPs requires the capability to reconstruct the system state $\boldsymbol{x}_i$ and the value of $f(\boldsymbol{x}_i)$ using the output $\boldsymbol{y}_i$ from each sensor $i$, we introduce the data set $\mathbb{D}_i$ contains $M_i \in \mathbb{N}$ samples $\{(\bm{\xi}^{(\iota)}, \varphi^{(\iota)} )\}$ for $\iota = 1,\cdots,M_i$ described in the following assumption.
	\begin{assumption} \label{assumption_dataset}
		A scalar state $\varphi^{(\iota)}$ is an observation of $f(\cdot)$ measured at the state $\bm{\xi}^{(\iota)} \in \mathbb{R}^n$ for $\iota \in \mathbb{I}_{[1,M]}$, i.e., $\varphi^{(\iota)} \!=\! f(\bm{\xi}^{(\iota)}) \!+\! w^{(\iota)}$. And $w^{(\iota)}$ is the related observation noise bounded by a constant $\bar{w} \in \mathbb{R}_{>0}$, i.e., $| w^{(\iota)} | \le \bar{w}$ for $\forall \iota \in \mathbb{N}_0$.
	\end{assumption}
	Assumption~\ref{assumption_dataset} admits the presence of noise in the training samples, particularly $\varphi^{(\iota)}$. This mirrors the real-world scenario, where $\varphi^{(\iota)}$ is influenced by environmental disturbances and uncertainties. 
	The upper bound is derivable from the sensor measurements in \eqref{eqn_measurement} by considering the bounded nature of measurement noise, which is elucidated later in Theorem~\ref{theorem_dataset_noise}.

	Notably, extracting the information from the measurement $\bm{y}_i$ is inaccessible for some sensors, even though $(\bm{A}, \bm{C}_i)$ is observable due to the unknown function $f$.
	In order to address this issue, we present a data collection strategy for building the training data set for GPs as discussed in Section~\ref{subsection_data_collection}. 
	Before that, we first present an estimated function for $f(\cdot)$ using the data set $\mathbb{D}$ with complete states. 
	GP predicts the value of $f(\cdot)$ and point $\bm{\xi}$ as a Gaussian distribution with posterior mean $\mu(\cdot)$ and variance $\sigma^2(\cdot)$ as
	\begin{align}
		&\mu(\bm{\xi}) = \bm{k}_{\mathbb{D}}^T(\bm{\xi}) (\bm{K}_{\mathbb{D}} + \bar{w}^2 \bm{I}_M)^{-1} \bm{\varphi}, \\
		&\sigma^2(\bm{\xi}) = \kappa(\bm{\xi}, \bm{\xi}) - \bm{k}_{\mathbb{D}}^T(\bm{\xi}) (\bm{K}_{\mathbb{D}} + \bar{w}^2 \bm{I}_M)^{-1} \bm{k}_{\mathbb{D}}(\bm{\xi}),
	\end{align}
	where $\bm{k}_{\mathbb{D}}(\bm{\xi}) = [\kappa(\bm{\xi},\bm{\xi}^{(1)}), \cdots, \kappa(\bm{\xi},\bm{\xi}^{(M)})]^T \in \mathbb{R}^{M}$, $\bm{\varphi} = [\varphi^{1}, \cdots, \varphi^{M}]^T \in \mathbb{R}^{M}$ and $\bm{K}_{\mathbb{D}} = [\kappa(\bm{\xi}^{(i)},\bm{\xi}^{(j)})]_{i,j = 1,\cdots, M} \in \mathbb{R}^{M \times M}$.
	Therefore, the posterior mean $\mu(\cdot)$ is adopted to estimate $f(\cdot)$, and the posterior variance $\sigma^2(\cdot)$ quantifies the prediction error shown in the following lemma.
	\begin{lemma} [\!\!\! \cite{hashimoto2022learning}] \label{lemma_GP_error_bound}
		Assume the unknown function $f(\cdot)$ satisfying Assumption~\ref{assumption_f} and predicted by GP regression using data set $\mathbb{D}$ with $M$ training samples following Assumption~\ref{assumption_dataset}.
		Then, the prediction error is bounded as
		\begin{align}
			| \mu(\bm{\xi}) - f(\bm{\xi}) | \le \eta(\bm{\xi}) := \sqrt{\beta} \sigma(\bm{\xi}),  && \forall \bm{\xi} \in \mathbb{R}^n,
		\end{align}
		where $\beta = \Gamma^2 - \bm{\varphi}^T (\bm{K}_{\mathbb{D}} + \bar{w}^2 \bm{I}_M)^{-1} \bm{\varphi} + M$.
	\end{lemma}
	Lemma~\ref{lemma_GP_error_bound} establishes a theoretical error bound in Gaussian processes, a methodology widely applied in safety-critical contexts, as exemplified in \cite{zhang2023safety}. While the coefficient $\beta$ scales with the dataset size, given by $\beta \le \Gamma^2 + M$, the work in \cite{williams2006gaussian} assures a concurrent reduction in the posterior variance $\sigma^2(\cdot)$.

	\section{Cooperative Learning with Partial Measurements}
	\label{section_observer_based_cooperative_learning}
	
	Within this section, we initially present the data collection strategy in Section~\ref{subsection_data_collection}. 
	Subsequently, we propose the COIN-GP algorithm in Section~\ref{subsection_learning_observer}, followed by a rigorous analysis of its learning performance in Section~\ref{subsection_performance_analysis}, ensuring the prediction guarantee for safety-critical applications.
	
	\subsection{Data Collection Strategy} 
	\label{subsection_data_collection}
	Owing to the challenge induced by the combination of incomplete measurements $\bm{y}_i$ of states $\bm{x}_0$ from sensor $i \in \mathcal{V}$ and the unknown dynamics $f(\cdot)$, the observability requirement of $(\bm{A}, \bm{C}_i)$ is insufficient, resulting in the inadequacy of usable training data pairs. 
	However, local observability of $(\bm{A}, \bm{C}_i)$ of every agent is not required when agents do not maintain individual GP models and instead rely on predictions received from their neighbors through the cooperative learning approach described in~\cref{subsection_learning_observer}. 
	Moreover, while the work~\cite{Yang_CCC2024_Kernel} addresses only a specific case, we need to consider arbitrarily structured matrices $\bm{A}$, $\bm{b}$, and $\bm{C}_i$ for constructing the datasets.
	To solve this problem, a condition for sufficient data acquisition is introduced, along with a corresponding data collection strategy and an analysis of noise propagation. 
	We present the property of observability for the sensors as follows.
	\begin{lemma} \label{lemma_linear_combination_CAn}
		Consider a system \eqref{eq_dynamics} with states $\bm{x}_0 \in \mathbb{R}^n$, which is measured by sensor $i$ following \eqref{eqn_measurement} with the output dimension $p_i$.
		If $(\bm{A}, \bm{C}_i)$ is observable, then there exists a set of matrices $\{\bm{H}_{i,d}\}_{d\in \mathbb{I}_{[0,n-1]}}$ such that
		\begin{align} \label{eqn_CAn_CAd}
			\bm{C}_i \bm{A}^n = \sum_{d=0}^{n-1} \bm{H}_{i,d} \bm{C}_i \bm{A}^d.
		\end{align}
	\end{lemma}
	\begin{IEEEproof}
		See the supplementary material.
	\end{IEEEproof}
	
	Considering the system dynamics defined in \eqref{eq_dynamics} along with the measurement function in \eqref{eqn_measurement}, the output $\bm{y}_i(k+n)$ for all $k \in \mathbb{N}$ is  expressed as
	\begin{align} \label{eqn_ykn}
		\bm{y}_i(k+n) &= \bm{C}_i \bm{A}^n \bm{x}_0(k) + \bm{v}_i(k+n) \nonumber \\
		&\quad + \bm{C}_i \sum_{d=0}^{n-1} \bm{A}^{n-d-1} \bm{b} f(\bm{x}_0(k+d)).
	\end{align}
	Combining the outcomes in Lemma~\ref{lemma_linear_combination_CAn},  \cref{eqn_ykn} is further reformulated as
	\begin{align}
		\bm{y}_i(k+n) &= \sum_{d=0}^{n-1} \bm{H}_{i,d} \bm{C}_i \bm{A}^d \bm{x}_0(k) + \bm{v}_i(k+n) \nonumber\\
		&\quad+ \bm{C}_i \sum_{d=0}^{n-1} \bm{A}^{n-d-1} \bm{b} f(\bm{x}_0(k+d)),
	\end{align}
	which is equivalent to
	\begin{align}
		\bm{y}_i(k+n) &= \sum_{d=0}^{n-1} \bm{H}_{i,d} \big( \bm{y}_i(k+d) - \bm{v}_i(k+d) \big) +\bm{v}_i(k+n) \nonumber \\
		&\quad- \sum_{d=1}^{n-1} \bm{H}_{i,d} \bm{C}_i \sum_{l=0}^{d-1} \bm{A}^{d-l-1} \bm{b} f(\bm{x}_0(k+l)) \nonumber \\
		&\quad+ \bm{C}_i \sum_{d=0}^{n-1} \bm{A}^{n-d-1} \bm{b} f(\bm{x}_0(k+d)).
	\end{align}
	Due to the fact that the double summation\footnote{The proof of \cref{eq_double_summation} is provided in the supplementary material.} 
	\begin{align}
		\label{eq_double_summation}
		&\sum_{d=1}^{n-1} \bm{H}_{i,d} \bm{C}_i \Big( \sum_{l=0}^{d-1} \bm{A}^{d-l-1} \bm{b} f(\bm{x}_0(k+l)) \Big) \nonumber\\
		& = \sum_{d=0}^{n-1} \bigg( \sum_{l=d+1}^{n-1} \bm{H}_{i,l} \bm{C}_i  \bm{A}^{l-d-1} \bigg) \bm{b}f(\bm{x}_0(k+d)),
	\end{align}
	then $\bm{y}_i(k+n)$ is written as
	\begin{align} \label{eqn_ykn_2}
		&\bm{y}_i(k+n) = \sum\limits_{d=0}^{n-1} \bm{H}_{i,d} \big(\bm{y}_i(k+d) \!-\! \bm{v}_i(k+d) \big)\!+ \!\bm{v}_i(k+n)  \nonumber\\
		&+\!\! \sum\limits_{d=0}^{n-1} \!\Big(\! \bm{C}_i\bm{A}^{n-d-1} \!\!-\!\!\! \sum\limits_{l=d+1}^{n-1} \!\! \bm{H}_{i,l} \bm{C}_i \bm{A}^{l-d-1} \Big) \bm{b} f(\bm{x}_0(k\!+\!d)). \!\!
	\end{align}
	For facilitating the construction of the training dataset and notional simplicity, we define an auxiliary state $\tilde{\bm{y}}_i(k)$ from measurements as 
	\begin{align}\label{eq_tilde_y}
		\tilde{\bm{y}}_i(k) = \bm{y}_i(k+n) - \sum_{d=0}^{n-1} \bm{H}_{i,d} \bm{y}_i(k+d)
	\end{align}
	and an auxiliary vector from matrices as 
	\begin{align}\label{eq_rho}
		\bm{\rho}_{i,d} = \big( \bm{C}_i\bm{A}^{n-d-1} - \sum_{l=d+1}^{n-1} \bm{H}_{i,l} \bm{C}_i \bm{A}^{l-d-1} \big) \bm{b}.
	\end{align}
	Based on the defined auxiliary variables \eqref{eq_tilde_y} and \eqref{eq_rho}, then one has
	\begin{align} \label{eqn_tilde_y}
		\tilde{\bm{y}}_i(k) &= \sum_{d=0}^{n-1} \Big(\bm{\rho}_{i,d} f(\bm{x}_0(k+d)) - \bm{H}_{i,d} \bm{v}_i(k+d) \Big) \nonumber\\
		&\quad+ \bm{v}_i(k+n).
	\end{align}
	
	To assess the feasibility of constructing a training dataset conducive to the inference of the unknown function $f$ and the estimated state vector $\boldsymbol{x}_0$ from measurements $\bm{y}_i$, we introduce the concept of collectability for a system. This concept indicates the sensor's capability in acquiring a training dataset $\mathbb{D}_i$ associated with the sensor $i \in \mathcal{V}$ that satisfies Assumption~\ref{assumption_dataset}, which is described in the following definition.
	
	\begin{definition}\label{def_collect}
		The collectability of the training dataset of a system $i\in\mathcal{V}$ is affirmed if and only if its corresponding dataset $\mathbb{D}_i$ can be established by obtaining measurements $\bm{y}_i$ and satisfies \cref{assumption_dataset}.
	\end{definition}
	
	Considering Definition~\ref{def_collect}, the sufficient condition of the collectability of a system is given as follows. 
	\begin{lemma} \label{lemma_data_collectable}
		Consider the system \eqref{eq_dynamics} measured by sensor $i \in \mathcal{V}$ following \eqref{eqn_measurement}.
		If $(\bm{A}, \bm{C}_i)$ is observable and there exists one $d^*_i \in \mathbb{I}_{[0, n-1]}$ such that
		\begin{align} \label{eqn_data_collection_condition}
			\bm{\rho}_{i,d} \begin{cases}
				\ne \bm{0}_{p_i \times 1}, & \text{if}~ d = d^*_i \\
				= \bm{0}_{p_i \times 1}, & \text{otherwise}
			\end{cases}, ~~d\in\mathbb{I}_{0,n-1},
		\end{align}
		then the data set $\mathbb{D}_i$ satisfying Assumption~\ref{assumption_dataset} is obtainable and constructible.
	\end{lemma}
	\begin{IEEEproof}
		See the supplementary material.
	\end{IEEEproof}
	Specifically, with $\bm{t}_i \in \mathbb{R}^{p_i}$ satisfying $\bm{t}_i^T \bm{\rho}_{i,d^*_i} \ne 0$ and $\bm{T}_i \in \mathbb{R}^{n \times n p_i}$ rendering $\bm{T}_i \boldsymbol{O}_{i}$ invertible, the training pair is constructed from the measurements as
	\begin{align}
		\varphi_i(k) &= \big( \bm{t}_i^T \bm{\rho}_{i,d^*_i} \big)^{-1} \bm{t}_i^T \tilde{\bm{y}}_i(k - d^*_i) = f(\bm{x}_0(k)) + w_{i,1}(k), \label{eqn_phi_main} \\
		\bm{\xi}_i(k) &= \big( \bm{T}_i \boldsymbol{O}_{i} \big)^{-1} \bm{T}_i \big( \bm{Y}_i(k) - \boldsymbol{G}_{i} \bm{\varphi}_i(k) \big) = \bm{x}_0(k) + \bm{w}_{i,2}(k), \label{eqn_xi_main}
	\end{align}
	where $\bm{Y}_i(k) = [\bm{y}_i^T(k), \cdots, \bm{y}_i^T(k+n-1)]^T$, $\bm{\varphi}_i(k) = [\varphi_i(k), \cdots, \varphi_i(k+n-2)]^T$, $\boldsymbol{G}_{i}$ is the block lower-triangular Toeplitz matrix whose $(r,c)$-th block is $\bm{C}_i \bm{A}^{r-c-1} \bm{b}$ for $r > c$ and $\bm{0}_{p_i \times 1}$ otherwise, and the transformed noises $w_{i,1}(k)$, $\bm{w}_{i,2}(k)$ collect the measurement noise contributions.
	It is worth mentioning that the sensor measurements and the training datasets of individual systems are not transmitted among the interconnected subsystems, thereby alleviating communication overhead. Moreover, despite the inability of all systems to assemble the data set $\mathbb{D}_i$, cooperative learning mechanisms are leveraged to mitigate this limitation. Through cooperative learning, neighboring systems contribute to compensating for the unavailability of data by sharing predictions derived from GP models trained on non-empty datasets, which is elaborated in Section~\ref{subsection_learning_observer}.

	\begin{remark}\label{remark_special_cases}
		Condition \eqref{eqn_data_collection_condition} indicates the auxiliary variable $\tilde{\bm{y}}(k)$ can be written as the function of $f(\bm{x}(k+d_i^*))$, and is irrelevant to the unknown function $f(\bm{x}(\cdot))$ and system state $\bm{x}$ at other time steps, which means the value of $f(\bm{x}(k))$ can be recovered using only $\tilde{\bm{y}}(k-d_i^*)$ as shown in \eqref{eqn_phi_main}.
		Note that this condition depends on the selection of $\bm{H}_{i,d}$, which may have multiple choices when $p_i > 1$.
		For sensors whose valid $\bm{H}_{i,d}$ for collectability check is hard to obtain, we assume they are not collectable and apply the cooperative learning strategy to enhance their prediction performance.
		
		The geometric picture is that observability ensures the observation space ``fully captures'' the state space dynamics, making the redundancy encoding via $\{\bm{H}_{i,d}\}$ possible. 
		For each row $j \in \mathbb{I}_{[1,p_i]}$ of $\bm{C}_i \bm{A}^n$, solve the linear system:
		$$
		\bm{g}_j \bm{C}_i \bm{A}^n = \sum_{d=0}^{n-1} \sum_{q=1}^{p_i} h_{d,q}^{j,i} \bm{g}_q \bm{C}_i \bm{A}^d.
		$$
		Find coefficients $\{h_{d,q}^{j,i}\}$ such that the $j$-th row of $\bm{C}_i \bm{A}^n$ is expressed as a linear combination of rows from $\boldsymbol{O}_i$. Stack these into matrices $\bm{H}_{i,d} \in \mathbb{R}^{p_i \times p_i}$ as defined in the proof.
		
		Since the data collection condition in Lemma~\ref{lemma_data_collectable} is a sufficient condition, to find such $\bm{H}_{i,d}$ with $d = 0, \cdots, n-1$ satisfying \eqref{eqn_data_collection_condition} is nontrivial.
		Nevertheless, the search is not exhaustive. For a fixed candidate index $d^*_i$, both \eqref{eqn_CAn_CAd} and the vanishing conditions $\bm{\rho}_{i,d} = \bm{0}$, $d \ne d^*_i$, are linear in the entries of $\{\bm{H}_{i,d}\}$, so collectability reduces to a linear feasibility problem with $n p_i^2$ unknowns and $p_i (2n-1)$ equations over the $n$ candidates $d^*_i$, which can be solved before deployment.
		Since this problem is overdetermined for $p_i = 1$ and underdetermined for $p_i \ge 2$, multi-output sensors generically admit a valid $\{\bm{H}_{i,d}\}$, whereas scalar-output sensors are collectable only for structured $\bm{C}_i$.
		In this light, we discuss some special cases.
		\begin{itemize}
			\item For single output measurement with $p_i = 1$, the matrices $\bm{H}_{i,d}$ degraded to scalar are unique, since $\boldsymbol{O}_{i} \in \mathbb{R}^{n \times n}$.
			Then, the evaluation of \eqref{eqn_data_collection_condition} is fast for unique $H_{i,d}$, but no freedom is left to enforce \eqref{eqn_data_collection_condition}.
			\item For a discrete-time canonical system in control and observation with sampling time $T \in \mathbb{R}_+$ as
			\begin{align}
				&\bm{A} = \begin{bmatrix}
					\bm{I}_{n-1} & \bm{0}_{(n-1) \times 1} \\
					\bm{0}_{1 \times (n-1)} & 0
				\end{bmatrix} + \begin{bmatrix}
					\bm{0}_{(n-1) \times 1} & T \bm{I}_{n-1} \\
					0 & \bm{0}_{1 \times (n-1)}
				\end{bmatrix}, \nonumber \\
				&\bm{b} = [\bm{0}_{1 \times (n-1)}, 1]^T, \bm{C}_i = [1,\bm{0}_{1 \times (n-1)}], \nonumber
			\end{align}
			it has $\bm{C}_i \bm{A}^d \bm{b} = 0$ for $d = 0, \cdots, n-2$ and $\bm{C}_i \bm{A}^{n-1} \bm{b} = 1$ satisfying \eqref{eqn_data_collection_condition} without considering $\bm{H}_{i,d}$.
		\end{itemize}
	\end{remark}

	\begin{remark}\label{remark_delay}
		Considering \eqref{eq_tilde_y} and \eqref{eqn_phi_main}, obtaining $\bm{\varphi}_i(k)$ for computing $\bm{\xi}_i(k)$ via \eqref{eqn_xi_main} requires access to measurements $\bm{y}_i(k), \cdots, \bm{y}_i(k + 2n - 1)$. Consequently, under the worst-case scenario, acquiring the training data pair $\{\bm{\xi}_i(k), \varphi_i(k)\}$ demands a sequential time interval of $k + 2n - 1$ steps. However, when $\boldsymbol{C}_i$ is an identity matrix, it is straightforward to derive that $\bm{\xi}_i(k)=\boldsymbol{y}_i(k)$ and $\varphi_i(k)= \boldsymbol{b}^T (\boldsymbol{y}(k+1)-\boldsymbol{A}\boldsymbol{y}(k))/(\boldsymbol{b}^T \boldsymbol{b})$.
	\end{remark}

	Considering the process of derivation of Lemma~\ref{lemma_data_collectable}, for a system $i\in\mathcal{V}$ satisfying \eqref{eqn_data_collection_condition}, the approach for data collection is summarized in Algorithm~\ref{algorithm_data_collection}.
	
	\begin{algorithm} [t]
		\caption{Data acquisition on subsystem $i$}
		\label{algorithm_data_collection}
		\begin{algorithmic}[1]
			\STATE Condition: $(\bm{A}, \bm{C}_i)$ is observable satisfying \eqref{eqn_data_collection_condition};
			\STATE Initialization: $\mathbb{D}_i = \emptyset$, $\bm{Y}_i = \emptyset$, $\bm{\varphi}_i = \emptyset$, $k = 0$;
			\WHILE{Terminate condition for collection is not satisfied}
			\STATE Collect $\bm{y}(k)$;
			\IF{$| \bm{Y}_i | < n$} 
			\STATE $\bm{Y}_i \leftarrow [ \bm{Y}_i, \bm{y}_i(k) ]$ 
			\ELSE
			\STATE $\bm{Y}_i \leftarrow [\bm{Y}_i(p_i+1:n p_i); \bm{y}_i(k) ]$;
			\STATE Calculate $\varphi_i(k - n + d_i^*)$ from \eqref{eqn_phi_main};
			\STATE $\bm{\varphi}_i \leftarrow [\bm{\varphi}_i, \varphi_i(k - n + d_i^*)]$
			\IF{$| \bm{\varphi}_i | \ge n - 1$}
			\STATE Calculate $\bm{\xi}_i(k - n + 1)$ from \eqref{eqn_xi_main};
			\STATE $\mathbb{D}_i \leftarrow \{ \mathbb{D}_i, \{\bm{\xi}_i(k - n + 1), \varphi_i(k - n + d_i^*)\} \}$;
			\ENDIF
			\ENDIF
			\STATE $k \leftarrow k + 1$
			\ENDWHILE
		\end{algorithmic}
	\end{algorithm}
	
	\begin{remark}
		Algorithm~\ref{algorithm_data_collection} allows online learning for each agent satisfying \eqref{eqn_data_collection_condition} by adding new samples into the training data set, which continuously improves the observation and prediction performance of $\bm{x}_0$ and $f(\cdot)$ during operation as detailed in Section~\ref{subsection_learning_observer}.
		Since one valid training pair requires measurements spanning up to $2n-1$ steps, the newest sample in $\mathbb{D}_i$ lags the current state by up to $2n-1$ steps; as $f(\cdot)$ is time-invariant, such delayed pairs remain valid samples in the sense of Assumption~\ref{assumption_dataset}, and the delay only postpones the arrival of samples without altering their distribution along the trajectory of $\bm{x}_0$.
		Consequently, the delay affects only the posterior variance $\sigma_i(\bm{x}_i(k))$ in the error bound of Section~\ref{subsection_performance_analysis}, which grows at most linearly with the distance travelled by $\bm{x}_0$ within $2n-1$ steps, and hence with $n$ for fast-varying systems, when $\bm{x}_0$ enters a region not yet covered by data, while the influence of the delay vanishes in revisited regions, since the GP exploits the entire data set rather than the most recent sample.
		The ultimate boundedness of the observation and prediction errors in Theorem~\ref{theorem_tracking_error_bound} is unaffected, as it holds for any $\mathbb{D}_i$ satisfying Assumption~\ref{assumption_dataset}, including an empty one.
	\end{remark}
	
	Given the collectible dataset, we now investigated the noise bound of the measurements for $f$ in the following theorem.
	\begin{theorem} \label{theorem_dataset_noise}
		Considering the dynamical system \eqref{eq_dynamics} and the sensors $i \in \mathcal{V}$ equipped with measurement systems described by \eqref{eqn_measurement}, the collection of the dataset $\mathbb{D}_i$ satisfies the conditions in Lemma~\ref{lemma_data_collectable}.
		The dataset $\mathbb{D}_i$ is collected through Algorithm~\ref{algorithm_data_collection} for GP prediction of $f(\cdot)$ under Assumption~\ref{assumption_f}, then Assumption~\ref{assumption_dataset} holds for $\mathbb{D}_i$ with the noise $w_i$, which is bounded by
		\begin{align} \label{eqn_upperbound_w}
			| w_i(k) | \le \bar{w}_i = \bar{w}_{i,1} + L_f \sqrt{\bar{w}_{i,2}}, && \forall k \in \mathbb{N},
		\end{align}
		where $L_f = \sqrt{2 L_{\kappa}} \Gamma$, and
		\begin{align}
			&\bar{w}_{i,1} = \big| \bm{t}_i^T \bm{\rho}_{i,d^*_i} \big|^{-1} \| \bm{t}_i \| \Big( 1 + \sum_{d=0}^{n-1} \| \bm{H}_{i,d} \| \Big) \bar{v}_i, \\
			&\bar{w}_{i,2} = \| \big( \bm{T}_i \boldsymbol{O}_{i} \big)^{-1} \bm{T}_i \| \big( \| \boldsymbol{G}_{i} \| \sqrt{n\!-\!1} \bar{w}_{i,1} + \sqrt{n} \bar{v}_i \big).
		\end{align}
		with $\bm{t}_i$ and $\bm{T}_i$ chosen according to \eqref{eqn_phi_main} and \eqref{eqn_xi_main}, respectively.
	\end{theorem}
	\begin{IEEEproof}
		See the supplementary material.
	\end{IEEEproof}

	\subsection{Observer-based Cooperative Learning}
	\label{subsection_learning_observer}
	To address the potential absence of some training datasets for GP models, we introduce the COIN-GP approach for the estimated function of the system $i \in \mathcal{V}$, expressed as
	\begin{align}
		\label{eqn_oberver_based_prediction}
		\hat{f}_i(k+1) &= \gamma_1 \sum_{j = 1}^N \tilde{a}_{i,j}(k) ( \hat{f}_i(k) - \hat{f}_j(k) )  \\
		&+ \gamma_2 \varpi_{i}(k) (\hat{f}_i(k) - \mu_i(\bm{x}_i(k))) + \mu_i(\bm{x}_i(k+1)), \nonumber
	\end{align}
	where $\gamma_1, \gamma_2 \in \mathbb{R}$ are estimation gains to be designed later, $\hat{f}_i(k)$ represents the estimation of $f(\boldsymbol{x}_i(k))$, $\forall k \in \mathbb{N}_{0}$ for simplicity.
	\begin{align}
		\label{eq_stateEsitmation}
		\bm{x}_i(k+1) = \bm{A} \bm{x}_i(k) + \bm{b} \hat{f}_i(k) + \bm{L}_i (\bm{C}_i \bm{x}_i(k)  \!- \! \bm{y}_i(k)). 
	\end{align}
	
	\begin{remark}
		The communication graph is assumed undirected because the proposed framework relies on bidirectional information exchange between neighboring agents. 
		This requirement can be relaxed by utilizing a directed graph in which each agent only broadcasts its own predictions, though this typically leads to degraded performance. 
	\end{remark}

	Specifically, $\bm{L}_i \in \mathbb{R}^{n \times p_i}$ denotes the observer gain matrix selected such that $\bm{A} + \bm{L}_i \bm{C}_i$ is Schur, that is, all eigenvalues lie strictly within the unit circle in the complex plane. 
	The design of $\bm{L}_i$ is performed locally at each node, so that the observer gain design is implemented in a fully distributed manner.
	If $(\bm{A}, \bm{C}_i)$ is observable, such a matrix $\bm{L}_i$ always exists and can be obtained using Ackermann's formula with given desired eigenvalues.
	The local prediction from the GP model for system $i$ is denoted by $\mu_i(\cdot)$, leveraging the individual dataset $\mathbb{D}_i$ if available. 
	The gains $\gamma_1, \gamma_2 \in \mathbb{R}$ in \eqref{eqn_oberver_based_prediction} are designed constants to balance the weights of the cooperative and individual learning. 
	The consensus weights $\tilde{a}_{i,j}$ and $\varpi_{i}$ are designed as
	\begin{align} \label{eqn_consensus_weight}
		\tilde{a}_{i,j}(k) \!=\! \frac{a_{i,j} \sigma_i^2(\bm{x}_i(k))}{\sigma_j^2(\bm{x}_j(k)) \!+\! \bar{w}_i^2}, && \varpi_{i}(k) \!=\! \frac{\sigma_i^2(\bm{x}_i(k))}{\sigma_{f,i}^{2}},
	\end{align}
	where $a_{i,j}$ is the entry of the adjacent matrix $\boldsymbol{\mathcal{A}}$ at the $i$-th row and $j$-th column.
	Therefore, the computation outlined in \eqref{eqn_oberver_based_prediction} solely relies on neighboring information. 
	Consequently, a distributed computation framework is realized, wherein the scalability of the proposed methodology is attainable, rendering it well-suited for large-scale sensor networks.
	Moreover, \eqref{eqn_consensus_weight} ensure bounded consensus weights, i.e., $\tilde{a}_{i,j}(k) \in \{0, \sigma_{f,i}^{2} / \bar{w}_i^2 \}$ and $\varpi_{i}(k) \in (0,1]$.
	In the early exploration phase, all posterior variances remain close to the priors, so \eqref{eqn_consensus_weight} reduces to uniform-weight cooperation.
	As data accumulate asymmetrically, a data-rich agent $i$ assigns a small weight $\tilde{a}_{i,j}$ to a data-poor neighbor $j$ while the latter keeps a large weight $\tilde{a}_{j,i}$, so that on a connected graph the predictions of agents without data follow those of data-rich agents, directly or through intermediate agents.
	
	\begin{remark}
		The evaluation of the proposed estimator in \eqref{eqn_oberver_based_prediction} on agent $i \in \mathcal{V}$ only requires the information exchange for the estimated prediction $\hat{f}_j(k)$ and posterior variance $\sigma_j^2(\bm{x}_j(k))$ from $j \in \mathcal{N}_i$.
		Without sharing the training samples or the local prediction $\mu_j(\bm{x}_j(k))$, the proposed method in \eqref{eqn_oberver_based_prediction} protects the data privacy of the systems.
		Moreover, the proposed method saves the computation of $\mu_j(\bm{x}_i)$ required for the conventional aggregation strategy used in MAS as \cite{yang_CDC2021_Distributed,dai_EJC2024_Decentralized,dai_TNNLS2025_Cooperative}.
	\end{remark}
	
	Throughout this work, the inter-agent communication is assumed synchronous and delay-free: at every sampling instant, each agent exchanges its current prediction and posterior variance with its neighbors before the updates are evaluated.
	For imperfect timing, COIN-GP can be integrated with delay-compensation strategies~\cite{Doostmohammadian_SCL2025_Distributed}, asynchronous distributed GP updates~\cite{Yang_AAAI2025_Asynchronous}, or dynamic average consensus~\cite{Lederer_TAC2023_Cooperative}, which is left for future work.

	\subsection{Performance Analysis}
	\label{subsection_performance_analysis}
	To analyze the estimation and identification performance for system states $\bm{x}_0$ and dynamics $f(\cdot)$, we define the their associated prediction and observation error for $f(\cdot)$ and $\bm{x}_0$ as
	\begin{align}
		e_{i}^f(k) = \hat{f}_i(k) - f(\bm{x}_0(k)), && \bm{e}_{i}^s(k) = \bm{x}_i(k) - \bm{x}_0(k),
	\end{align}
	respectively.
	Combining the system dynamics~\eqref{eq_dynamics} and output observation~\eqref{eqn_measurement} with estimated function~\eqref{eqn_oberver_based_prediction}, the dynamics of prediction and observation error are reformulated as
	\begin{align}
		e_{i}^f(k+1) &= \gamma_1 \sum_{j = 1}^N \tilde{a}_{i,j}(k) ( e_{i}^f(k) - e_{j}^f(k) )  \\
		&\quad+ \gamma_2 \varpi_{i,i}(k) (e_{i}^f(k) + \tilde{\mu}_i(k)) - \tilde{\mu}_i(k+1), \nonumber \\
		\bm{e}_i^s(k+1) &= \bm{A} \bm{x}_i(k) + \bm{b} \hat{f}_i(k) + \bm{L}_i (\bm{C}_i \bm{x}_i(k) - \bm{y}_i(k)) \nonumber \\
		&\quad - \bm{A} \bm{x}_0(k) - \bm{b} f(\bm{x}_0(k))  \\
		&= ( \bm{A} + \bm{L}_i \bm{C}_i ) \bm{e}_i^s(k) + \bm{b} e_{i}^f(k+1) - \bm{L}_i \bm{v}_i(k), \nonumber
	\end{align}
	with $\tilde{\mu}_i(k) = f(\bm{x}_0(k)) - \mu_i(\bm{x}_i(k))$.
	Moreover, we define the concatenated prediction and observation error as $\bm{e}^f(k) = [e_{1}^f(k), \cdots, e_{N}^f(k)]^T$ and $\bm{e}^s(k) = [\bm{e}_1^s(k)^T, \cdots, \bm{e}_N^s(k)^T]^T$ respectively, whose dynamics are written as
	\begin{align} \label{eqn_concatenated_error_dynamics}
		\bm{e}^f(k+1) &=\big( \gamma_1 \bm{\mathcal{L}}(k) + \gamma_2 \bm{\mathcal{B}}(k) \big) \bm{e}^f(k)    +\gamma_2 \bm{\mathcal{B}}(k) \tilde{\bm{\mu}}(k) \nonumber\\
		&\quad - \tilde{\bm{\mu}}(k+1),  \\
		\label{eqn_concatenated_error_dynamics2} \bm{e}^s(k+1) &=  \bm{\Theta} \bm{e}^s(k) + (\bm{I}_N \otimes \bm{b}) \bm{e}^f(k) - \bm{\Upsilon} \bm{v}(k), 
	\end{align}
	with the concatenated prediction error denotes by $\tilde{\bm{\mu}}(k) = [\tilde{\mu}_1(k), \cdots, \tilde{\mu}_N(k)]^T$, and the concatenated measurement noise written as $\bm{v}(k) \!=\! [\bm{v}_1^T\!(k), \!\cdots\!, \bm{v}_N^T\!(k)]^T$, where 
	\begin{align*}
		&\bm{\Sigma}(k) = \mathrm{diag}(\sigma_1^2(\bm{x}_1(k)), \cdots, \sigma_N^2(\bm{x}_N(k)) ), \\
		&\bar{\bm{\Sigma}}(k) = \mathrm{diag}(\sigma_1^2(\bm{x}_1(k))+\bar{w}_1^2, \cdots, \sigma_N^2(\bm{x}_N(k))+\bar{w}_N^2 ),\\
		&\bm{\mathcal{L}}(k) = \bm{\mathcal{D}}(k) - \bm{\Sigma}(k) \bm{\mathcal{A}} \bar{\bm{\Sigma}}^{-1}(k), \bm{\Upsilon} = \mathrm{blkdiag}(\bm{L}_1, \cdots\!, \bm{L}_N),\\
		& \bm{\mathcal{D}}(k) = \mathrm{diag}(\bm{\Sigma}(k) \bm{\mathcal{A}} \bar{\bm{\Sigma}}^{-1}(k) \bm{1}_N),  \\
		&\bm{\mathcal{B}}(k) = \mathrm{diag}(\sigma_{f,1}^{-2}, \cdots, \sigma_{f,N}^{-2}) \bm{\Sigma}(k), \\
		& \bm{\Theta} = \mathrm{blkdiag}(\bm{A} + \bm{L}_1 \bm{C}_1, \cdots, \bm{A} + \bm{L}_N \bm{C}_N). 
	\end{align*}
	
	Given the error dynamics \eqref{eqn_concatenated_error_dynamics} and \eqref{eqn_concatenated_error_dynamics2} of prediction and observation of the distributed system, the joint learning performance for the distributed system is shown in the following theorem.
	\begin{theorem}
		\label{theorem_tracking_error_bound}
		Consider a system \eqref{eq_dynamics} measured by a distributed system \eqref{eqn_measurement} and bounded measurement noise $\bar{v}_i$ for $i \in \mathcal{V}$. 
		Each subsystem is observable and under the condition \eqref{eqn_data_collection_condition} with the training samples using Algorithm~\ref{algorithm_data_collection}, such that the collected data set $\mathbb{D}_i$ satisfies Assumption~\ref{assumption_dataset}.
		Use the observer-based cooperative learning strategy proposed in \eqref{eqn_oberver_based_prediction} satisfying Assumption\ref{assumption_f} and choose proper $\gamma_1 \in \mathbb{R}_{<0}$, $\gamma_2, \alpha_1 \in \mathbb{R}_{>0}$ and symmetric positive matrix $\bm{Q} \in \mathbb{R}^{2 \times 2}$ such that
		\begin{align}
			\lambda_f^* &= \max_{i = 1\cdots N, k \in \mathbb{N}} | \lambda_i(\gamma_1 \bm{\mathcal{L}}(k) \!+\! \gamma_2 \bm{\mathcal{B}}(k)) | < 1, \\
			\zeta_0 &= \underline{\lambda}(\bm{Q}) \!-\!  2 \alpha_1 \max(\lambda_f^*, \lambda_s^*) \max\big(\gamma_2^2 + 1, \| \bm{\Upsilon} \| ^ 2\big) \| \bm{P} \| \nonumber \\
			&\quad- \alpha_1^2 \max\big(\gamma_2^2 + 1, \| \bm{\Upsilon} \| ^ 2\big)^2 \| \bm{P} \| > 0,
		\end{align}
		with
		\begin{align}
			\lambda_s^* =& \max_{j = 1\cdots N, i = 1\cdots n} |\lambda_i(\bm{A} + \bm{L}_j \bm{C}_j)|, \\
			\bm{Q} =& -\begin{bmatrix}
				\lambda_f^* & 0 \\
				\| \bm{b} \| & \lambda_s^*
			\end{bmatrix}^T \bm{P} \begin{bmatrix}
				\lambda_f^* & 0 \\
				\| \bm{b} \| & \lambda_s^*
			\end{bmatrix} + \bm{P}, 
		\end{align}
		where the matrix $\bm{P}$ is the solution of discrete-time Lyapunov equation $\bm{\Phi}_1^T \bm{P} \bm{\Phi}_1 - \bm{P} = - \bm{Q}$.
		Then there exists $\bar{k} \in \mathbb{N}$ such that the observation and prediction error are ultimately bounded by
		\begin{align}
			\label{eqn_ultimate_bound_individual}
			|e_{i}^f(k)| \le \| \bm{e}^f(k) \| \le \bar{e}, && \| \bm{e}_i^s(k) \| \le \| \bm{e}^s(k) \| \le \bar{e},
		\end{align}
		respectively, for $\forall k > \bar{k}$, with the ultimate error bound $\bar{e}$ formulated as
		\begin{align}
			\bar{e} = \sqrt{{\bar{\lambda}(\bm{P})}/{\underline{\lambda}(\bm{P})}} \zeta_0^{-1} \big( \zeta_1 + \sqrt{ \zeta_1^2 + \zeta_0 \zeta_2} \big) \alpha_2,
		\end{align}
		where the coefficients $\zeta_1$ and $\zeta_2$ are defined by $\zeta_1 = \tau ( \lambda^*  + \alpha_1 \tau ) \| \bm{P} \|$ and $\zeta_2 = \tau^2 \| \bm{P} \|$. 
		The scalar $\alpha_2$ is defined as
		\begin{align}
			\alpha_2 = \alpha_1^{-1} L_f^2 \sqrt{N} + \| \bar{\bm{\eta}} \| + \| \bar{\bm{\eta}}^* \| + \| \bar{\bm{v}} \|,
		\end{align}
		where $\| \bar{\bm{v}} \| = \| [\bar{v}_1, \cdots, \bar{v}_N]^T \|$, the prediction accuracy related terms $\bar{\bm{\eta}}$ and $\bar{\bm{\eta}}^*$ are defined as $\bar{\bm{\eta}} = [\bar{\eta}_1, \cdots, \bar{\eta}_N]^T$ and $\bar{\bm{\eta}}^* = [\bar{\eta}_1^*, \cdots, \bar{\eta}_N^*]^T$ with $\bar{\eta}_i = \sup_{k \in \mathbb{N}_0} \eta_i(\bm{x}_i(k))$ and $\bar{\eta}_i^* = (\beta \sigma_{f,i}^2)^{-1} \bar{\eta}_i^3$, respectively. 
		Moreover, the constants $\lambda^* = \max(\lambda_f^*, \lambda_s^*)$ and $\tau = \max(\gamma_2^2 + 1, \| \bm{\Upsilon} \| ^ 2)$ denote the maximal singular values of $\bm{\Phi}_1$ and $\bm{\Phi}_2$ respectively, in which $\lambda_s^* = \max_{j = 1,\cdots,N} \max_{i = 1,\cdots,n} |\lambda_i(\bm{A} + \bm{L}_j \bm{C}_j)| \in [0,1)$ due to the block triangular structure of $\bm{\Lambda}$ and the designed $\bm{L}_i$.
	\end{theorem}
	\begin{IEEEproof}
		See the supplementary material.
	\end{IEEEproof}

	\begin{remark} \label{remark_gershgorin}
		Notably, the stability guarantees in Theorem~\ref{theorem_tracking_error_bound} rely on the global spectral properties of the graph, particularly the eigenvalues of $\gamma_1 \bm{\mathcal{L}}(k) + \gamma_2 \bm{\mathcal{B}}(k)$. While exact verification requires global topology information, practitioners can derive a conservative sufficient condition for ad-hoc networks using the Gershgorin Circle Theorem and Weyl's inequality. Since $\lambda_{max}(\bm{\mathcal{L}}) \leq 2N$ and $\lambda_{max}(\bm{\mathcal{B}}) \leq 1$, choosing parameters such as $|\gamma_1| < \frac{1}{4N}$ and $\gamma_2 < \frac{1}{2}$ easily satisfies the bound $\lambda_{max}(\gamma_1 \bm{\mathcal{L}} + \gamma_2 \bm{\mathcal{B}}) < 1$ without requiring full graph knowledge, at the price of a conservatism that grows with $N$: the gain is scaled by the worst case $4N$, whereas the spectrum of a sparse graph grows only with the node degree, so the resulting design converges more slowly and settles at a larger ultimate error than a globally tuned one. Nevertheless, extending the current framework to a fully distributed, topology-agnostic adaptive tuning mechanism, where agents do not even require knowledge of the network upper bounds (e.g., $N$), remains a highly relevant direction for future work.
	\end{remark}
	
	Theorem~\ref{theorem_tracking_error_bound} shows that the observation and prediction using the proposed observer-based cooperative learning in \eqref{eqn_oberver_based_prediction} will converge to a small area around the zero, achieving practical convergence defined in Definition~\ref{def_convergence}.
	The additional condition requiring $\zeta_0 > 0$ is practically not restrictive, which can be easily achieved by choosing sufficiently small $\alpha_1$.
	However, smaller $\alpha_1$ induces larger ultimate bound $\bar{e}$ by considering the inverse relationship between $\alpha_1$ and $\alpha_2$.
	Unlike the setting in continuous-time without measurement noise such as in \cite{umlauftFeedbackLinearizationBased2020}, increasing the control gains by choosing eigenvalues of $\gamma_1 \bm{\mathcal{L}}(k) + \gamma_2 \bm{\mathcal{B}}(k)$ or $\bm{A} + \bm{L}_i \bm{C}_i$ for $i \in \mathcal{V}$ closer to $0$ may not tighter the error bound $\bar{e}$.
	This is because the induced high gains increase the sensitivity of the multi-sensor system to the model uncertainty and measurement noise reflected by a larger singular value of $\bm{\Phi}_2$.
	
	\section{Simulations}\label{sec_simulation}
	In \cref{subsec_simSetting}, we describe the simulation setting. The proposed approach is then compared with prevailing distributed GP frameworks and an adaptive neural-network baseline through Monte Carlo tests in \cref{subsec_comparison,subsec_MCT}, and its robustness under unreliable communication is examined in \cref{subsec_dropout}.
	
	\subsection{Simulation Setting}
	\label{subsec_simSetting}
	We consider a system represented by \eqref{eq_dynamics} with state dimension $n = 2$.\footnote{The full simulation configuration is given in the supplementary material.} In this case, the systems state is denoted as $\bm{x}_0 = [x_{0,1}, x_{0,2}]^T$, and the system matrices are defined as $\bm{A} = \begin{bmatrix}
		1 & 1 \\ 0 & 0
	\end{bmatrix}$ and $\boldsymbol{b}=[0~1]^T$.
	The state trajectory, which the distributed system aims to estimate, is given by $x_{0,1}(k) = \arctan(a_1 k) \sin(a_2 k)$, where $a_1$ and $a_2 \in \mathbb{R}$ chosen as $0.01$ and $0.05$, respectively.
	The overall COIN-GP framework is illustrated in \cref{fig_common}. 
	The simulation is conducted over a maximum time step of $k=500$.

	\begin{figure}
		\centering
		\includegraphics[width=0.72\linewidth]{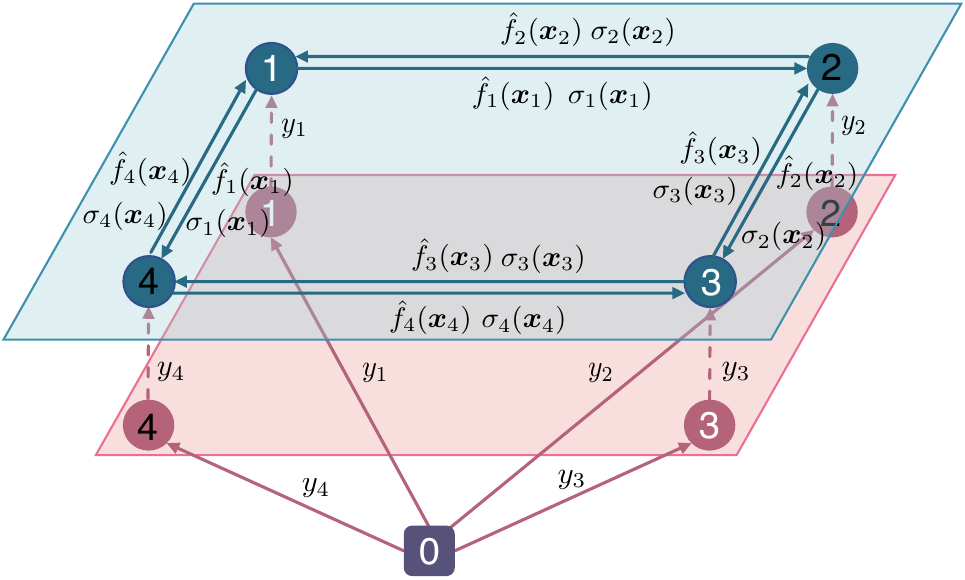}
		\caption{Communication topology among sensors/agents. 
			Teal nodes represent the GP model instances, and magenta nodes represent the actual sensing nodes that transmit measurements from system node 0.
			The black numbers on the nodes indicate they lack individual training datasets, i.e., the node $2$ and $4$, while the white numbers imply the existence of datasets on these nodes, i.e., node $1$ and $3$.}
		\label{fig_common}
	\end{figure}
	
	In order to estimate the system states $\bm{x}_0$ and unknown dynamics $f(\cdot)$, the distributed system consisting of $N = 4$ sensors with an undirected topology shown in \cref{fig_common} is employed.
	Each subsystem follows \eqref{eqn_measurement} with noise variance $\bar{v}_i = 0.001, \forall i \in \mathcal{V}$. The specific configurations for each sensor are detailed in Table \ref{table_sensor_configuration}.
	Note that $\bm{L}_i$ is chosen according to $\lambda_{1,2}(\bm{A} \!+\! \bm{L}_i \bm{C}_i)$ using robust pole placement \cite{kautsky1985robust}.
	For agents $2$ and $4$, data collection is deliberately deactivated, so that their data sets remain empty and the matrices $\bm{H}_{i,d}$, $\bm{t}_i$ and $\bm{T}_i$ are left unspecified in \cref{table_sensor_configuration}.
	This setting serves to demonstrate the efficacy of the proposed cooperative learning approach, which does not require all agents to collect data or participate in modeling the unknown function.
	For agent $1$ and $3$, the training data set $\mathbb{D}_1$ and $\mathbb{D}_3$ are collected online using Algorithm \ref{algorithm_data_collection} satisfying Assumption \ref{assumption_dataset}.
	Additionally, each Gaussian process model established is initialized at $k = 0$ with an empty data set, i.e., $\mathbb{D}_i = \emptyset$, and prior knowledge regarding the kernel is chosen as automatic relevance determination squared exponential kernel with $\kappa_i(\bm{\xi}, \bm{\xi}') = \sigma_{f,i}^2 \exp \big(-\frac{1}{2} \sum_{d = 1}^n {| \xi_d - \xi_d' |^2}/{l_d^2} \big)$, where $\bm{\xi} = [\xi_1, \xi_2]^T$. 
	The hyperparameters are set to $\sigma_{f,i} = 0.017$, $l_1 = 1.8$ and $l_2 = 0.025$, which are obtained through hyperparameter optimization using the non-empty dataset $\mathbb{D}_1$ and $\mathbb{D}_3$ with $k = 500$. 
	
	\begin{table}[t]
		\caption{Sensor configurations} \label{table_sensor_configuration}
		\centering
		\scriptsize
		\setlength{\tabcolsep}{3pt}
		\begin{tabular}{l c c c c}
			\hline
			Properties & Sensor 1 & Sensor 2 & Sensor 3 & Sensor 4 \\
			\hline
			$\bm{C}_i$ in \eqref{eqn_measurement}					& $[1,\!0]$		& $[1,\!0;2,\!1]$	& $[1,\!1;1,\!2]$					& $[0,\!1;3,\!2]$	\\
			$\bm{H}_{i,0}$ in \eqref{eqn_CAn_CAd}					& $[0]$			& $-$				& $0.5[1,\!0;1,\!0]$				& $-$				\\
			$\bm{H}_{i,1}$ in \eqref{eqn_CAn_CAd}					& $[1]$			& $-$				& $0.5[0,\!1;0,\!1]$				& $-$				\\
			$\bm{t}$ in \eqref{eqn_phi_main}								& $1$			& $-$				& $[2;1]$							& $-$				\\
			$\bm{T}$ in \eqref{eqn_xi_main}								& $\bm{I}_2$	& $-$				& $[\bm{I}_2,\bm{0}_{2\times2}]$	& $-$				\\
			$\lambda_{1,2}(\bm{A} \!+\! \bm{L}_i \bm{C}_i)$	& $[0.4,\!0.5]$	& $[0.4,\!-\!0.5]$	& $[0.5,\!-\!0.3]$					& $[0.4,\!-\!0.3]$	\\
			\hline
		\end{tabular}
	\end{table}
	
	\subsection{Performance comparison}\label{subsec_comparison}
	To demonstrate the superiority of the proposed observer-based online cooperative learning \eqref{eqn_oberver_based_prediction} with the devised state estimation method \eqref{eq_stateEsitmation}, we conduct a comparative simulation for analyzing the observation and prediction performance against the prevailing distributed GP frameworks. It's crucial to note that due to the existing methods rely on complete observations necessitating full measurement of system states $\bm{x}_0$ and are unsuitable for output-based control scenarios. In order to ensure a fair comparison, we employ the same observer \eqref{eq_stateEsitmation}, across all compared approaches.
	The details of our experimental setup are elucidated below.
	The compared methods are the proposed cooperative learning \eqref{eqn_oberver_based_prediction} with $\gamma_2 = 1$ and the state-dependent gain $\gamma_{1,i}(k) = -\gamma_2 \varpi_i(k) / \sum_{j \in \mathcal{N}_i} \tilde{a}_{i,j}(k)$, under which the consensus term in \eqref{eqn_oberver_based_prediction} reduces to the weighted average of the neighbors' estimates with the weights \eqref{eqn_consensus_weight}; local learning \cite{nguyen2008local} with $\hat{f}_i(\bm{x}_i(k)) = \mu_i(\bm{x}_i(k))$; MoE \cite{tresp2000mixtures} with $\hat{f}_i(\bm{x}_i(k)) = \sum_{j \in \{i, \mathcal{N}_i \}} a_{i,j} \mu_j(\bm{x}_j(k))$; and PoE \cite{yang_CDC2021_Distributed} with
	\begin{align*}
		\hat{f}_i(\bm{x}_i(k)) = \frac{ \sum_{j \in \{i, \mathcal{N}_i \}} a_{i,j} \sigma_j^{-2}(\bm{x}_j(k)) \mu_j(\bm{x}_j(k))}{\sum_{j \in \{i, \mathcal{N}_i \}} a_{i,j} \sigma_j^{-2}(\bm{x}_j(k))}.
	\end{align*}
	Moreover, with $a_{i,i} = 1$, the shorthands $\mu_j = \mu_j(\bm{x}_j(k))$ and $\sigma_j^2 = \sigma_j^2(\bm{x}_j(k))$, and the log-variance weight $\rho_j(k) = \frac{1}{2} \ln \big( \sigma_{f,j}^2 / \sigma_j^2 \big)$, we evaluate GPoE \cite{Bach_ICML2015_Distributed} with
	\begin{align*}
		\hat{f}_i(\bm{x}_i(k)) = \frac{ \sum_{j \in \{i, \mathcal{N}_i \}} a_{i,j} \rho_j(k) \sigma_j^{-2} \mu_j}{\sum_{j \in \{i, \mathcal{N}_i \}} a_{i,j} \rho_j(k) \sigma_j^{-2}},
	\end{align*}
	BCM \cite{trespBayesianCommitteeMachine2000} with
	\begin{align*}
		\hat{f}_i(\bm{x}_i(k)) \!=\! \frac{ \sum_{j \in \{i, \mathcal{N}_i \}} a_{i,j} \sigma_j^{-2} \mu_j}{\sum\limits_{j \in \{i, \mathcal{N}_i \}} \!a_{i,j} \sigma_j^{-2} \!+\! \big( 1 \!-\!\! \sum\limits_{j \in \{i, \mathcal{N}_i \}} \!a_{i,j} \big) \sigma_{f,i}^{-2}},
	\end{align*}
	and RBCM \cite{Bach_ICML2015_Distributed} with
	\begin{align*}
		\hat{f}_i(\bm{x}_i(k)) \!=\! \frac{ \sum_{j \in \{i, \mathcal{N}_i \}} a_{i,j} \rho_j(k) \sigma_j^{-2} \mu_j}{\sum\limits_{j \in \{i, \mathcal{N}_i \}} \!\!a_{i,j} \rho_j(k) \sigma_j^{-2} \!+\! \big( 1 \!-\!\!\! \sum\limits_{j \in \{i, \mathcal{N}_i \}} \!\!a_{i,j} \rho_j(k) \big) \sigma_{f,i}^{-2}}.
	\end{align*}
	Finally, to assess the trade-off between capacity and formal guarantees, an adaptive radial-basis-function neural network baseline is evaluated, replacing each agent's GP by an RBFNN with linear output weights,
	\begin{align*}
		&\hat{f}_i^{\mathrm{NN}}(\bm{x}) = \hat{\bm{w}}_i^T \bm{\phi}_i(\bm{x}), \\
		&[\bm{\phi}_i(\bm{x})]_m = \exp \Big( \! -\frac{1}{2} \sum_{d=1}^{n} \frac{(x_d - c_{m,d}^i)^2}{s_d^2} \Big), \; m = 1, \dots, M_{\phi},
	\end{align*}
	with $M_{\phi} = 100$ Gaussian features whose centers $\bm{c}_m^i$ are assigned randomly over the operating region, since no prior data is available to place them, mirroring the a-priori fixed GP kernel hyperparameters, and whose widths $s_d$ equal the GP lengthscales.
	Upon each reconstructed pair $(\bm{\xi}_i, \varphi_i)$, the weights are updated online by normalized least-mean-squares with $\sigma$-modification,
	\begin{align*}
		\hat{\bm{w}}_i \leftarrow \hat{\bm{w}}_i + \eta\, \frac{\bm{\phi}_i(\bm{\xi}_i) \big( \varphi_i - \hat{\bm{w}}_i^T \bm{\phi}_i(\bm{\xi}_i) \big)}{\varepsilon + \| \bm{\phi}_i(\bm{\xi}_i) \|^2} - \eta \sigma_m \hat{\bm{w}}_i,
	\end{align*}
	with $\eta = 0.5$, $\sigma_m = 10^{-4}$, $\varepsilon = 10^{-8}$, and a replay buffer sharing the GP data budget, under the identical observer, data-collection pipeline, topology, and noise realizations.
	RBFNN-Local is non-cooperative, while RBFNN-Coop applies the consensus law \eqref{eqn_oberver_based_prediction} with $\mu_i(\cdot)$ replaced by the NN prediction; since the NN provides no posterior variance, the weights \eqref{eqn_consensus_weight} degenerate to uniform weights ($\tilde{a}_{i,j} = a_{i,j}$, $\varpi_i = 1$, gains re-tuned to $\gamma_1 = -0.05$, $\gamma_2 = 0.2$).
	All NN hyperparameters and consensus gains are grid-searched in favor of the baseline.
	All sensors are initialized at $k = 0$ with same state estimation for each methods, specifically $\bm{x}_1(0) = [0.6294, 0.0406]^T$, $\bm{x}_2(0) = [-0.7460, 0.0413]^T$, $\bm{x}_3(0) = [0.2647, -0.0402]^T$ and $\bm{x}_4(0) = [-0.4430, 0.0047]^T$.

	On the individual-agent level, the benefit of cooperation concentrates on the agents without data collectability: over the $100$ Monte Carlo runs of \cref{subsec_MCT}, the per-agent mean prediction error of agents $2$ and $4$ drops from $0.067$ (Local) to $0.031$ (COIN-GP), while for the data-collectable agents COIN-GP performs on par with or better than local learning; MoE instead weighs empty and trained models equally, deteriorating the accuracy of the reliable agents.
	\cref{figure_Prediction_Error_Bound} shows the prediction error from GP, which is bounded by its theoretical upper bound derived in \cref{lemma_GP_error_bound}.
	Moreover, the RKHS norm bound $\Gamma = 5$ in computing $\eta(\cdot)$ is estimated by using a data-driven method provided in \cite{hashimoto2022learning}.
	
	\begin{figure}[t]
		\centering
		\begin{tikzpicture}
			\def\file{fig/Prediction_ErrorBound.txt}
			\begin{semilogyaxis}[xlabel={$k$},ylabel={GP Error},
				xmin=1, ymin = 2e-5, xmax = 499,ymax=9e-1,legend columns=4,
				width=\columnwidth,height=3.5cm,legend pos=north east,
				]
				
				\addplot[carrot_orange, thick]    table[x = t_set , y  = PredictionError_set ]{\file};
				\addplot[carrot_orange, thick, dashed]    table[x = t_set , y  = PredictionErrorBound_set ]{\file};
				
				\legend{
					$\| \bm{\mu}(\bm{x}) \!-\! \bm{1}_N \!\otimes\! f(\bm{x}_0) \|$, Error bound
				}
			\end{semilogyaxis}
		\end{tikzpicture}
		\vspace{-0.3cm}
		\caption{
			Prediction errors of local GP models and their theoretical bounds.
		}
		\label{figure_Prediction_Error_Bound}
	\end{figure}
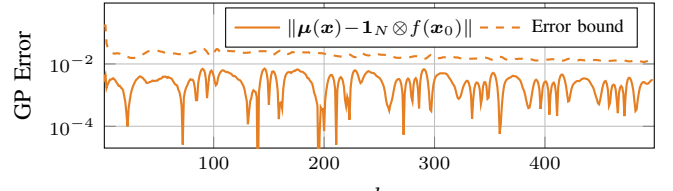

	\subsection{Monte Carlo Test}\label{subsec_MCT}
	To show the generalizability of our method, the Monte Carlo tests are conducted $100$ times for each method.
	The initial observer states $\bm{x}_i(0)$ for $i \in \mathcal{V}$ are sampled randomly from a uniform distribution in $[-1,1] \times [-0.05,0.05]$, and the system dynamics is also randomized with $a_1$ following a uniform distribution between $0.01$ to $0.05$ and $a_2$ satisfying a uniform distribution between $0.05$ to $0.1$.
	To reflect realistic long-term networked deployments, the main comparison is conducted with the streaming implementation, i.e., sliding-window datasets with budget $M = 20$ ensuring a bounded per-step complexity, under mild communication unreliability (per-step link failure probability $p = 0.2$).
	In addition to the GP-based baselines, the adaptive RBFNN baselines introduced in the experimental setup are included in the same figures and tables.
	
	As illustrated in \cref{figure_Observation_Error}, where solid lines denote the mean and shaded areas the $\pm 1$ standard deviation (shown for COIN-GP, PoE, MoE, and Local; the bands of the newly added baselines are omitted for legibility), COIN-GP achieves the lowest observation and prediction errors among all nine methods, including the strongest uncertainty-weighted aggregators (GPoE, BCM) and the RBFNN baselines, while RBCM diverges under link failures in this streaming setting and leaves the displayed range.
	The same ranking holds for every link failure probability (\cref{subsec_dropout}) and for every data budget $M \in \{20, 100, 500\}$. As reported in the last column of \cref{table_observation_error}, one full COIN-GP time step takes $0.33\,$ms on a single CPU core, on par with purely local learning ($0.34\,$ms) and only marginally above the static and RBFNN baselines ($0.28$--$0.31\,$ms).

	\begin{figure}[t]
		\centering
		\begin{tikzpicture}
			\begin{groupplot}[group style={group size=1 by 2, vertical sep=0.25cm,
					x descriptions at=edge bottom},
				width=\columnwidth, height=3.7cm, ymode=log,
				xmin=1, xmax=499,
				tick label style={font=\scriptsize},
				legend style={font=\scriptsize, /tikz/every even column/.append style={column sep=3pt}},
				legend columns=5,
				legend image code/.code={\draw[mark repeat=2, mark phase=2, #1] plot coordinates {(0cm,0cm) (0.22cm,0cm)};}]
				\nextgroupplot[ylabel={$\| \bm{x} \!-\! \bm{1}_N \!\otimes\! \bm{x}_0 \|$},
				ymin=2e-2, ymax=1.2, legend to name=mclegend]
				\def\file{fig/Observation_Error_Time.txt}
				\bandcurve{carrot_orange, very thick}{carrot_orange!50}{\file}{t_set_Obs}{ErrorObs_DoCoL_NoDis}
				\bandcurve{carrot_blue, thick}{carrot_blue!50}{\file}{t_set_Obs}{ErrorObs_DoPOE_NoDis}
				\bandcurve{carrot_yellow, thick}{carrot_yellow!50}{\file}{t_set_Obs}{ErrorObs_DoMOE_NoDis}
				\bandcurve{carrot_green, thick}{carrot_green!50}{\file}{t_set_Obs}{ErrorObs_NoCoL_NoDis}
				\addplot[c_purple, each nth point=2] table[x=t_set_Obs, y=mean_ErrorObs_DoGPOE_NoDis]{\file};
				\addplot[cyan, densely dashed, each nth point=2] table[x=t_set_Obs, y=mean_ErrorObs_DoBCM_NoDis]{\file};
				\addplot[c_red, densely dotted, thick, each nth point=2] table[x=t_set_Obs, y=mean_ErrorObs_DoRBCM_NoDis]{\file};
				\addplot[black!70, dashed, each nth point=2] table[x=t_set_Obs, y=mean_ErrorObs_NoNN_NoDis]{\file};
				\addplot[black!40, dotted, thick, each nth point=2] table[x=t_set_Obs, y=mean_ErrorObs_DoNN_NoDis]{\file};
				\legend{COIN-GP, PoE, MoE, Local, GPoE, BCM, RBCM, RBFNN-Local, RBFNN-Coop}
				\nextgroupplot[xlabel={$k$}, ylabel={$\| \hat{\bm{f}}(\bm{x}) \!-\! \bm{1}_N f(\bm{x}_0) \|$},
				ymin=1e-2, ymax=0.6]
				\def\file{fig/Prediction_Error_Time.txt}
				\bandcurve{carrot_orange, very thick}{carrot_orange!50}{\file}{t_set_Pre}{ErrorPre_DoCoL_NoDis}
				\bandcurve{carrot_blue, thick}{carrot_blue!50}{\file}{t_set_Pre}{ErrorPre_DoPOE_NoDis}
				\bandcurve{carrot_yellow, thick}{carrot_yellow!50}{\file}{t_set_Pre}{ErrorPre_DoMOE_NoDis}
				\bandcurve{carrot_green, thick}{carrot_green!50}{\file}{t_set_Pre}{ErrorPre_NoCoL_NoDis}
				\addplot[c_purple, each nth point=2] table[x=t_set_Pre, y=mean_ErrorPre_DoGPOE_NoDis]{\file};
				\addplot[cyan, densely dashed, each nth point=2] table[x=t_set_Pre, y=mean_ErrorPre_DoBCM_NoDis]{\file};
				\addplot[c_red, densely dotted, thick, each nth point=2] table[x=t_set_Pre, y=mean_ErrorPre_DoRBCM_NoDis]{\file};
				\addplot[black!70, dashed, each nth point=2] table[x=t_set_Pre, y=mean_ErrorPre_NoNN_NoDis]{\file};
				\addplot[black!40, dotted, thick, each nth point=2] table[x=t_set_Pre, y=mean_ErrorPre_DoNN_NoDis]{\file};
			\end{groupplot}
		\end{tikzpicture}
		
		\ref{mclegend}
		\caption{State observation (top) and prediction (bottom) errors over $100$ Monte Carlo runs ($M = 20$, $p = 0.2$).}
		\label{figure_Observation_Error}
	\end{figure}
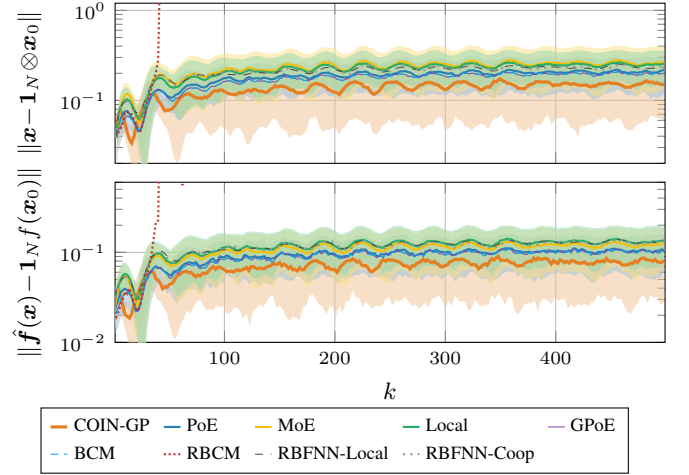
	For a clear comparison, we have numerically listed several evaluation metrics for both observation error and prediction error in \cref{table_observation_error}. To mitigate the impact of initial errors, all values are calculated using data starting from $k= 100$.
	COIN-GP attains the best mean, median, and RMSE for both the observation and the prediction error, while RBCM diverges under link failures in the streaming setting (denoted div.).
	Therefore, the effectiveness and superiority of the proposed method are clearly demonstrated.
	
	Regarding the neural-network baseline, three observations substantiate the capacity-versus-guarantee trade-off: the RBFNN matches the local GP almost exactly, confirming comparable capacity; uniform-weight cooperation provides the NN with at best marginal benefit, since without calibrated uncertainty the consensus incorporates the predictions of agents with empty datasets, whereas the variance-weighted fusion of COIN-GP reduces the error of local learning by $40\%$; and the GP prediction is accompanied by the deterministic bound of \cref{lemma_GP_error_bound} at every time step (cf.\ \cref{figure_Prediction_Error_Bound}), whereas no comparable bound exists for the RBFNN error.
	
	\begin{table}[t]
		\caption{Observation and prediction errors and wall-clock time per step using different methods ($k \geq 100$, $100$ runs, $M = 20$, $p = 0.2$, single CPU core)} \label{table_observation_error}
		\centering
		\scriptsize
		\setlength{\tabcolsep}{4pt}
		\begin{tabular}{l c c c c c c c}
			\hline
			& \multicolumn{3}{c}{Observation error} & \multicolumn{3}{c}{Prediction error} & Time \\
			Methods & Mean & Median & RMSE & Mean & Median & RMSE & [ms] \\
			\hline
			COIN-GP     & $\bm{0.144}$ & $\bm{0.129}$ & $\bm{0.170}$ & $\bm{0.075}$ & $\bm{0.066}$ & $\bm{0.089}$ & $0.33$ \\
			PoE         & $0.197$ & $0.193$ & $0.214$ & $0.099$ & $0.097$ & $0.110$ & $0.29$ \\
			MoE         & $0.248$ & $0.257$ & $0.278$ & $0.115$ & $0.120$ & $0.130$ & $0.29$ \\
			Local       & $0.237$ & $0.233$ & $0.257$ & $0.125$ & $0.127$ & $0.138$ & $0.34$ \\
			GPoE        & $0.187$ & $0.185$ & $0.199$ & $0.096$ & $0.094$ & $0.105$ & $0.30$ \\
			BCM         & $0.187$ & $0.185$ & $0.199$ & $0.096$ & $0.094$ & $0.105$ & $0.30$ \\
			RBCM        & div. & $0.928$ & div. & div. & $0.103$ & div. & $0.31$ \\
			RBFNN-Coop  & $0.226$ & $0.218$ & $0.248$ & $0.123$ & $0.119$ & $0.137$ & $0.28$ \\
			RBFNN-Local & $0.223$ & $0.210$ & $0.240$ & $0.124$ & $0.119$ & $0.134$ & $0.28$ \\
			\hline
		\end{tabular}
	\end{table}

	\subsection{Unreliable Communication}\label{subsec_dropout}
	To assess the robustness of the cooperative mechanism under degraded connectivity, every communication link fails independently at each time step with probability $p \in \{0, 0.2, 0.4, 0.6, 0.8\}$; a failed link removes the neighbor from the consensus term of COIN-GP and from the aggregation of the static baselines at that step, using the streaming setting of the main comparison ($M = 20$).
	
	\begin{figure}[t]
		\centering
		\begin{tikzpicture}
			\begin{groupplot}[group style={group size=2 by 1, horizontal sep=0.45cm},
				width=0.48\columnwidth, height=3.6cm,
				xlabel={$p$}, xmin=0, xmax=0.8, xtick={0,0.2,0.4,0.6,0.8},
				restrict y to domain*=0:0.5,
				scaled y ticks=false, yticklabel style={/pgf/number format/fixed},
				tick label style={font=\scriptsize}, label style={font=\small},
				legend style={font=\scriptsize}, legend columns=4,
				grid=both, grid style={gray!20}]
				\nextgroupplot[ylabel={observation error}, ymin=0.13, ymax=0.27,
				legend to name=droplegend]
				\addplot[carrot_orange, very thick, mark=*] coordinates
				{(0,0.1363) (0.2,0.1445) (0.4,0.1565) (0.6,0.1715) (0.8,0.2053)};
				\addplot[carrot_blue, thick, mark=square*] coordinates
				{(0,0.1927) (0.2,0.1990) (0.4,0.2058) (0.6,0.2139) (0.8,0.2220)};
				\addplot[carrot_yellow, thick, mark=diamond*] coordinates
				{(0,0.2557) (0.2,0.2496) (0.4,0.2441) (0.6,0.2389) (0.8,0.2342)};
				\addplot[carrot_green, thick, mark=x] coordinates
				{(0,0.2385) (0.2,0.2385) (0.4,0.2385) (0.6,0.2385) (0.8,0.2385)};
				\addplot[c_purple, thick, mark=triangle*] coordinates
				{(0,0.1806) (0.2,0.1879) (0.4,0.1971) (0.6,0.2079) (0.8,0.2200)};
				\addplot[cyan, thick, densely dashed, mark=triangle] coordinates
				{(0,0.1805) (0.2,0.1878) (0.4,0.1970) (0.6,0.2079) (0.8,0.2200)};
				\addplot[c_red, thick, mark=asterisk] coordinates
				{(0,0.2219) (0.2,21527) (0.4,32308) (0.6,33364) (0.8,23418)};
				\legend{COIN-GP, PoE, MoE, Local, GPoE, BCM, RBCM}
				\nextgroupplot[ylabel={prediction error}, ymin=0.06, ymax=0.14,
				ylabel near ticks, yticklabel pos=right]
				\addplot[carrot_orange, very thick, mark=*] coordinates
				{(0,0.0678) (0.2,0.0753) (0.4,0.0841) (0.6,0.0941) (0.8,0.1136)};
				\addplot[carrot_blue, thick, mark=square*] coordinates
				{(0,0.0927) (0.2,0.1001) (0.4,0.1068) (0.6,0.1130) (0.8,0.1185)};
				\addplot[carrot_yellow, thick, mark=diamond*] coordinates
				{(0,0.1152) (0.2,0.1161) (0.4,0.1178) (0.6,0.1197) (0.8,0.1214)};
				\addplot[carrot_green, thick, mark=x] coordinates
				{(0,0.1261) (0.2,0.1261) (0.4,0.1261) (0.6,0.1261) (0.8,0.1261)};
				\addplot[c_purple, thick, mark=triangle*] coordinates
				{(0,0.0888) (0.2,0.0969) (0.4,0.1046) (0.6,0.1116) (0.8,0.1182)};
				\addplot[cyan, thick, densely dashed, mark=triangle] coordinates
				{(0,0.0886) (0.2,0.0968) (0.4,0.1046) (0.6,0.1116) (0.8,0.1182)};
				\addplot[c_red, thick,  mark=asterisk] coordinates
				{(0,0.1155) (0.2,13106) (0.4,19866) (0.6,20581) (0.8,14340)};
			\end{groupplot}
		\end{tikzpicture}
		
		\ref{droplegend}
		\caption{Mean observation (left) and prediction (right) errors ($k \geq 100$, $20$ Monte Carlo runs, $M = 20$) versus the link failure probability $p$.}
		\label{figure_dropout}
	\end{figure}
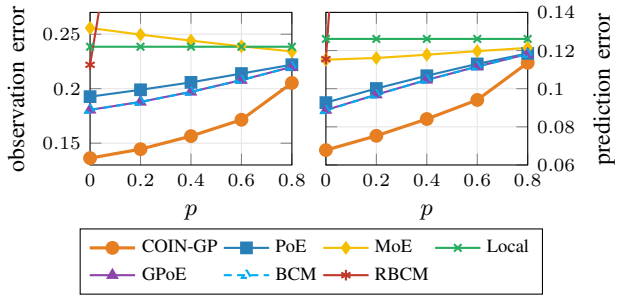
	
	As shown in \cref{figure_dropout}, COIN-GP is the most accurate method for every $p$, since its recursion retains the last consensus estimate whenever a neighbor is unavailable, whereas the memoryless aggregators approach the non-cooperative baseline as $p$ grows.
	The dashed BCM curve is close to GPoE in \cref{figure_dropout}, with mean errors differing by less than $2 \times 10^{-4}$ for every $p$. The prior-variance experts of the non-collecting agents are exactly cancelled by the BCM prior-correction term and receive zero weight $\rho_j$ in GPoE, so both aggregators fuse the same informative experts and differ only in their weighting.
	RBCM diverges for every $p > 0$ (mean errors above $10^3$), as its prior-correction denominator vanishes whenever the effective neighborhood contains only near-prior experts, a situation that link failures readily create under tight budgets, confirming that the robustness of COIN-GP stems from its uncertainty-aware recursive structure.

	\section{Conclusion and Discussion}\label{sec_conclusion}
	This paper introduced COIN-GP, an observer-based cooperative learning framework that combines distributed state estimation with Gaussian process regression for systems with known linear dynamics, an unknown nonlinear component, and partial state measurements.
	The framework consists of three elements, namely a data collection strategy together with sufficient conditions under which a sensor is able to reconstruct training data from its own incomplete measurements, an uncertainty-weighted consensus law that lets sensors without collectable data benefit from the models of their neighbors, and a joint analysis of the coupled estimation and learning loop that yields a single error bound covering both the state estimate and the model prediction.
	The central conceptual message is that the posterior variance of the Gaussian process is not merely an accuracy indicator but the quantity that renders cooperation meaningful, since it tells each sensor how much of a neighbor's prediction to trust. It is precisely this calibrated weighting, which is absent both in aggregation schemes without a probabilistic model and in deterministic function approximators, that makes the interconnection of observer and learner analyzable and keeps the cooperative estimate reliable when parts of the network hold no data at all.
	
	Several limitations delineate the scope of the present results and suggest directions for future research.
	First, for long-term real-time deployment, the computational complexity of the Gaussian process must be prevented from growing unboundedly. 
	A sliding-window budget already bounds the per-step cost uniformly over time at a moderate accuracy loss, and more elaborate streaming techniques, such as logGP~\cite{Lederer_2021ICML_Gaussian} and SkyGP~\cite{Yang_Zhang_Dai_Yu_Zhang_Huang_Sadeghian_Haddadin_2026}, can be seamlessly embedded into the data-collecting agents.
	Second, the system matrices $(\boldsymbol{A},\boldsymbol{B})$ are assumed to be known.
	Bounded perturbations $(\Delta\bm{A}, \Delta\bm{B})$ split into a matched part in the range of $\bm{B}$, which is absorbed into $\bm{f}$ via $\tilde{\bm{f}}(\bm{x}) = \bm{f}(\bm{x}) + \bm{B}^\dagger (\Delta\bm{A} \bm{x} + \Delta\bm{B} \bm{f}(\bm{x}))$ and hence learned online, and an unmatched remainder, which acts as a bounded disturbance and enlarges the error bound $\bar{e}$ in proportion to its magnitude. A quantitative robustness analysis and the extension to fully unknown $(\boldsymbol{A},\boldsymbol{B})$ are left for future work.
	Third, the collectability condition \eqref{eqn_data_collection_condition} is sufficient but not necessary, and it is verified per sensor for a fixed structure of $\bm{H}_{i,d}$.
	Characterizing collectability by necessary and sufficient conditions, and exploiting the network structure so that sensors jointly satisfy a collectability condition that none of them meets individually, would considerably widen the class of admissible sensor configurations.
	Beyond these points, the framework is developed for estimation, whereas the joint error bound is exactly the ingredient required by safety-critical control; embedding COIN-GP into distributed predictive or safety-filter architectures is therefore a natural next step.

	\bibliographystyle{ieeetr}        
	\bibliography{ref}
	
\end{document}